%% file: main.tex
\def\year{2023}\relax
\documentclass[letterpaper]{article}
\usepackage{aaai23} 
\usepackage{times} 
\usepackage{helvet} 
\usepackage{courier} 
\usepackage[hyphens]{url} 
\usepackage{graphicx} 
\usepackage{graphicx} 
\usepackage{natbib} 
\usepackage{caption} 
\DeclareCaptionStyle{ruled}%
{labelfont=normalfont,labelsep=colon,strut=off}
\usepackage{comment}

\usepackage{amsmath}
\usepackage{amssymb}
\usepackage{color}
\usepackage{xspace}
\usepackage{todonotes}
\usepackage[normalem]{ulem}

\title{Splitting Answer Set Programs with respect to Intensionality Statements (Extended Version)}

\author{Jorge Fandinno\equalcontrib\ and Yuliya Lierler\equalcontrib}

\affiliations{University of Nebraska at Omaha\\\{jfandinno,ylierler\}@unomaha.edu}

\input{macro}

\begin{document}

\maketitle

\begin{abstract}
    Splitting a logic program allows us to reduce the
    task of computing its stable models to similar tasks
    for its subprograms. 
    This can be used to 
    increase solving performance and to prove the correctness of programs.
    We generalize the conditions under which this technique is applicable, by considering not only dependencies between predicates but also their arguments and context.
    This allows splitting  programs commonly used in practice to which previous results were not~applicable.
\end{abstract}

\input{introduction.tex}
\input{background.tex}

\input{here_and_there_lambda.tex}

\input{splitting_logic_programs.tex}

\input{splitting_theories.tex}

\section*{Acknowledgements}

We are thankful to Vladimir Lifschitz, Torsten Schaub and the anonymous reviewers for their comments on earlier versions of this draft.
The first author is partially supported by the Nebraska EPSCoR 95-3101-0060-402 grant.

\bibliography{krr,bib,procs}

\newpage\input{appendix.tex}

\end{document}

%% file: macro.tex
\newtheorem{theorem}{Theorem}
\newtheorem{proposition}{Proposition}
\newtheorem{lemma}{Lemma}

\newtheorem{corollary}{Corollary}
\providecommand{\qed}{~\hfill$\square$}%
\newenvironment{proof}{\trivlist\item\ignorespaces\noindent\itshape Proof.\normalfont}{\qed\endtrivlist}
\newenvironment{proofof}[1]{\trivlist\item\ignorespaces\noindent\bf Proof of #1.\normalfont}{\qed\endtrivlist}

\newcommand{\boldd}{\mathbf{d}}
\newcommand{\bolde}{\mathbf{e}}

\newcommand{\boldt}{\mathbf{t}}

\newcommand{\boldX}{\mathbf{X}}
\newcommand{\boldY}{\mathbf{Y}}
\newcommand{\boldZ}{\mathbf{Z}}

\newcommand{\boldx}{\mathbf{x}}

\newcommand{\boldr}{\mathbf{r}}
\newcommand{\bolds}{\mathbf{s}}

\newcommand{\tuple}[1]{\langle#1\rangle}

\newcommand{\pdef}{{\lambda}}
\newcommand{\pdefb}{{\beta}}
\newcommand{\pdefc}{{\gamma}}
\newcommand{\pdefs}{{\Lambda}}

\newcommand{\eqdef}{%
  \mathrel{\vbox{\offinterlineskip\ialign{%
    \hfil##\hfil\cr%
    $\scriptscriptstyle\mathrm{def}$\cr%
    \noalign{\kern1pt}%
    $=$\cr%
    \noalign{\kern-0.1pt}%
}}}}

\newcommand{\SPos}[1]{\ensuremath{\mathit{SP}(#1)}}
\newcommand{\Pnn}[2]{\ensuremath{\mathit{Pnn}_{#1}(#2)}}
\newcommand{\Nnn}[2]{\ensuremath{\mathit{Nnn}_{#1}(#2)}}

\DeclareRobustCommand{\modelsht}{%
  \mathrel{\mathchoice
   {{\models}_{\!\!\!\!\scriptscriptstyle\rm ht}}
   {{\models}_{\!\!\!\!\scriptscriptstyle\rm ht}}
   {{\models}_{\!\!\!\!\scriptscriptstyle\rm ht}}
   {{\models}_{\!\!\!\!\scriptscriptstyle\rm ht}}
}}

\newcommand{\Pos}[2]{\mathit{Sp}_{#1}({#2})}

\newcommand{\HH}{\mathcal{H}}
\newcommand{\XX}{\mathcal{X}}

\newcommand{\AAA}{\mathcal{A}}
\newcommand{\BB}{\mathcal{B}}
\newcommand{\CC}{\mathcal{C}}

\newcommand{\SSS}{\mathcal{S}}

\newcommand{\At}[1]{\mathit{At}(#1)}
\newcommand{\EM}[1]{\mathit{EM}(#1)}
\newcommand{\G}{\mathfrak{G}}
\newcommand{\Fun}[2]{\mathit{Fun}_{#1}({#2})}

\def\lrar{\leftrightarrow}

\def\on{\textit{on}}

\def\fo{first\nobreakdash-order\xspace}

%% file: introduction.tex
\section{Introduction}

Answer set programming (ASP; \citeauthor{lifschitz08b}~\citeyear{lifschitz08b}) is a declarative logic programming paradigm well-suited for solving knowledge-intensive search problems.
Its success relies on the combination of a rich knowledge representation language with efficient solvers for finding solutions to problems expressed in this language~\cite{lifschitz19a}.
Solutions for logic programs in ASP are called ``stable models''.
\citet{liftur94a} introduced a fundamental
{\em  splitting} method (or, simply, splitting) in the theory of ASP.
Often, splitting is
used to reduce the task of computing the stable
models of a logic program to the task of computing these of its subprograms.
Also, the splitting is used to understand the meaning of a program in terms of its smaller components and it has become a key instrument in constructing proofs of correctness for logic programs~\cite{cafali20a}.
The original condition required to split a logic program was generalized by~\citet{oikjan08a} in the context of propositional programs.
Also, when stable models were defined for arbitrary first\nobreakdash-order formulas~\cite{peaval05a,feleli07a} and infinitary propositional formulas~\cite{truszczynski12a},
the splitting method was extended to these new settings~\cite{felelipa09a,harlif16a}.
However, when dealing with first\nobreakdash-order formulas (logic programs with variables are often identified with first\nobreakdash-order formulas of a restricted syntactic form), the splitting condition is not sufficiently general for some applications.

For instance, consider the following simplified fragment of the blocks world encoding~\cite{lifschitz02a}:
\small
\begin{align}
\mathit{on}(B,L,T+1) &\leftarrow \mathit{on}(B,L,T),\ not\  \mathit{non}(B,L,T+1)
    \label{eq:blocks.inertia}
\\
\mathit{on}(B,L,T+1) &\leftarrow \mathit{move}(B,L,T)
    \label{eq:blocks.action}
\\
\mathit{non}(B,L',T) &\leftarrow \mathit{on}(B,L,T),\  location(L'),\ L \neq L'
    \label{eq:blocks.non}
\end{align}
\normalsize
We are interested in splitting this fragment so that one subprogram contains instances of these rules for all time points before a (positive) threshold~$t$ and another  contains rules corresponding to the time points
at or after~$t$. 
This form of splitting will help, for example,  understanding iterative solving of planning problems~\cite{gekakasc17a} and
arguing correctness of action descriptions (that frequently result in rules of the form presented in our blocks world example).
However, the proposed splitting is impossible under the currently available Splitting Theorem for first\nobreakdash-order formulas~\cite{felelipa09a}. The splitting relies on identifying ``non-circularly'' dependent rules of a program. 
\citet{felelipa09a} define dependencies 
in terms of predicate symbols.
In our example, $on/3$ depends on itself in rule~\eqref{eq:blocks.inertia}.

Here, we formulate a more general condition for the applicability of the splitting method that considers a refined version of the dependency graph that takes into account not only the predicate dependencies but also their arguments and context.
Among others, this makes it applicable to the discussed example.
The notion of \emph{intensionality statement} introduced in this work is of key importance.
This is a refinement of the idea of intensional and extensional predicates~\cite{feleli11a} originally stemming from databases.
It provides us with means for specifying arguments of the predicates for which these are seen as intensional or extensional. Thus, the same predicate symbol may be both intensional and extensional depending on its context.
This refinement is the basis for a new Splitting result.

The rest of the paper is organized as follows.
After reviewing some preliminaries, we introduce the notion of intensionality statement.
Then, we present the key ideas for the new Splitting Theorem in the context of logic programs.
Finally, we generalize this result to first\nobreakdash-order theories.

%% file: background.tex
\section{Preliminaries}

We start by reviewing a many-sorted first\nobreakdash-order language.
The use of a many\nobreakdash-sorted language is motivated by its ability to formalize commonly used features of logic programs such as arithmetic operations~\cite{falilusc20a,lifschitz21a} and aggregate expressions~\cite{fahali22a}.
We follow the presentation by~\citet[Appendix~A]{fanlif22a}.
After the presentation of many-sorted language, the review the logic of here-and-there follows. Syntax and semantics of disjunctive logic programs concludes this section.

\subsubsection{Many-sorted first-order theories.} \label{sec:smodels}
A (many-sorted) signature consists of symbols of three
kinds---\emph{sorts}, \emph{function constants}, and
\emph{predicate constants}.  A reflexive and transitive \emph{subsort}
relation is defined on the set of sorts.
A tuple $s_1,\dots,s_n$ ($n\geq 0$) of \emph{argument sorts} is assigned
to every function constant and to every predicate constant; in addition, a
\emph{value sort} is assigned to every function constant.
Function constants with $n=0$ are called \emph{object constants}.
We assume that for every sort, an infinite sequence of \emph{object
  variables} of that sort is chosen.  \emph{Terms} over a (many-sorted)
signature~$\sigma$ are defined recursively as usual.
\emph{Atomic formulas} over~$\sigma$ are either (i)
  expressions of the form $p(t_1,\dots,t_n)$, where~$p$ is a predicate
  constant and $t_1,\dots,t_n$ are terms such that their sorts are
  subsorts of the argument sorts~$s_1,\dots,s_n$ of~$p$, or (ii)
  expressions of the form $t_1=t_2$, where $t_1$ and~$t_2$ are terms such that
  their sorts have a common supersort.
\emph{(First-order) formulas} over signature~$\sigma$ are formed from atomic formulas
and the 0-place connective~$\bot$ (falsity) using the binary
connectives $\land$, $\lor$, $\to$ and the quantifiers $\forall$, $\exists$.
The other connectives are treated as abbreviations: $\neg F$ stands for
$F\to\bot$ and~$F\lrar G$ stands for $(F\to G)\land (G\to F)$.

An \emph{interpretation}~$I$ of a signature~$\sigma$ assigns
\begin{itemize}
  \item
  a non-empty \emph{domain} $|I|^s$ to every sort~$s$ of~$I$, so that
  $|I|^{s_1}\subseteq |I|^{s_2}$ whenever~$s_1$ is a subsort of~$s_2$,
  \item
a function~$f^I$ from $|I|^{s_1}\times\cdots\times|I|^{s_n}$ to $|I|^s$ to
every function constant~$f$ with argument sorts $s_1,\dots,s_n$ and
value sort~$s$, and
\item a Boolean-valued function~$p^I$ on~$|I|^{s_1}\times \cdots \times|I|^{s_n}$ to every predicate constant~$p$
with argument sorts $s_1,\dots,s_n$.
\end{itemize}
If~$I$ is an interpretation over a signature~$\sigma$ then
by~$\sigma^I$ we denote the signature
obtained from~$\sigma$ by adding, for every element~$d$ of domain~$|I|^s$,
its \emph{name}~$d^*$ as an object constant of sort~$s$. 
The interpretation~$I$ is extended to $\sigma^I$ by defining $(d^*)^I=d$.
The value $t^I$ assigned by an interpretation~$I$ of~$\sigma$ to a ground
term~$t$ over~$\sigma^I$ and the
 satisfaction relation (denoted, $\models$) between an interpretation of~$\sigma$ and a
sentence over~$\sigma^I$ are defined recursively, in the usual
way~\cite[Section~1.2.2]{lif08b}. An interpretation is called a {\em model} of a \emph{theory} -- a (possibly infinite)  set of sentences -- when it satisfies every sentence in this theory. 
Some of the examples considered contain integers, function constants such as $+$ and comparison predicate constants such as $\leq$ or $>$ (used utilizing infix notation common in arithmetic). We call these function and predicate constants {\em arithmetic}.  In these examples, we assume that (i) the underlying signature contains a sort integer and (ii) interpretations of special kind are considered so that they interpret arithmetic function  and predicate constants as customary in arithmetic (see, for example, treatment of the integer sort by~\citet{falilusc20a} and~\citet{lifschitz21a}).
If $\boldd$ is a tuple $d_1,\dots,d_n$ of domains elements of~$I$
then~$\boldd^*$ stands for the tuple $d_1^*,\dots,d_n^*$ of their names.
If $\boldt$ is a tuple $t_1,\dots,t_n$ of ground terms then~$\boldt^I$
is the tuple $t_1^I,\dots,t_n^I$ of values assigned to them
by~$I$.

\subsubsection{Here-and-there.}\label{ssec:ht}

The first\nobreakdash-order logic of here\nobreakdash-and\nobreakdash-there, introduced by \mbox{\citet{pea04,peaval05a}}, forms a monotonic base for stable model semantic~\cite{gellif88b,gellif91a}.
The many\nobreakdash-sorted case was recently studied by~\citet{fanlif22a}.
%
%
%
By~$\At{I}$ we denote the set of ground atoms of the form~$p(\boldd^*)$ such that~\hbox{$I \models p(\boldd^*)$}, where~$p$ is a predicate symbol and $\boldd$ is a tuple of elements of domains of~$I$.
An \emph{HT\nobreakdash-interpretation} of~$\sigma$ is a pair
$\langle \HH,I\rangle$, where $I$ is an interpretation
of~$\sigma$, and $\HH$ is a subset of $\At{I}$.
(In terms of Kripke models with two worlds,~$\HH$
describes the predicates in the here-world and~$I$ captures the there-world).
The satisfaction relation~$\modelsht$ between
HT\nobreakdash-interpretation $\langle \HH, I\rangle$ of~$\sigma$
and a sentence~$F$ over~$\sigma^I$ is defined recursively:
\begin{itemize}
\item
  $\langle \HH, I\rangle \modelsht p(\boldt)$,
  if $p((\boldt^I)^*)\in \HH$ 
\item
$\langle \HH, I\rangle \modelsht t_1=t_2$ if $t_1^I=t_2^I$;
\item
$\langle \HH, I\rangle \modelsht F\land G$ if
$\langle \HH, I\rangle \modelsht F$ and
$\langle \HH, I\rangle \modelsht G$;
\item
$\langle \HH, I\rangle \modelsht F\lor G$ if
$\langle \HH, I\rangle \modelsht F$ or
$\langle \HH, I\rangle \modelsht G$;
\item
  $\langle \HH, I\rangle \modelsht F\to G$ if
    (i) $\langle \HH, I\rangle \not\modelsht F$ or $\langle \HH, I\rangle \modelsht G$,
    and
(ii)
    $I \models F\to G$;
\item
  $\langle \HH, I\rangle\modelsht\forall X\,F(X)$
 if $\langle \HH, I\rangle\modelsht F(d^*)$
  for each \hbox{$d\in|I|^s$}, where~$s$ is the sort of~$X$;
\item
  $\langle \HH, I\rangle\modelsht\exists X\,F(X)$
 if $\langle \HH, I\rangle\modelsht F(d^*)$
  for some \hbox{$d\in|I|^s$}, where~$s$ is the sort of~$X$.
\end{itemize}
If $\langle \HH, I\rangle \modelsht F$ holds, we say that $\langle \HH, I\rangle$ \emph{satisfies}~$F$ and that $\langle \HH, I\rangle$ is an \emph{HT\nobreakdash-model} of~$F$.
If two formulas have the same HT\nobreakdash-models then we say that they are \emph{HT\nobreakdash-equivalent}. 

It is easy to see that
$(\At{I},I)$ is an HT\nobreakdash-model of a sentence $F$ whenever $I$ is a model of $F$.
About a model~$I$ of a theory~$\Gamma$, we say it is \emph{stable} if, for
every proper subset~$\HH$ of~$\At{I}$, HT\nobreakdash-interpretation
$\langle \HH,I \rangle$ does not satisfy~$\Gamma$.
In application to theories of a single sort, this definition is equivalent to the original definition by~\citet{pea04,peaval05a}.
In addition, if the theory is finite, then this definition of a stable model is also equivalent to the definition of such model in terms of the operator~SM \cite{feleli07a,feleli11a}, when all predicate are considered to be intensional.
We generalize the distinction between intensional and extensional predicates in the next section.

\subsubsection{Disjunctive logic programs.}

A \emph{disjunctive rule} is a formula of the form~$\mathit{Head} \leftarrow \mathit{Body}$ where~$\mathit{Head}$ is a list of atomic formulas and~$\mathit{Body}$ is a list of \emph{literals}, that is, atomic formulas possibly preceded by one or two occurrences of negation $not$.
A \emph{disjunctive program} is a (possibly infinite) set of disjunctive rules.
We identify each rule~${\mathit{Head} \leftarrow \mathit{Body}}$ with the universal closure of the formula~$B \to H$, where~$B$ is the conjunction of all literals in~$\mathit{Body}$ after replacing~$\mathit{not}$ by~$\neg$; and~$H$ is the disjunction of all atomic formulas in~$\mathit{Head}$.
We often write rules as formulas~$B \to H$ omitting the reference to the universal closure.
For instance, rule~\eqref{eq:blocks.inertia} is understood as the universal closure of formula
\small
\begin{gather}
  \begin{aligned}
    \mathit{on}(B,L,T) \wedge \neg\mathit{non}&(B,L,T) 
    \to  \mathit{on}(B,L,T+1).
  \end{aligned}
    \label{eq:blocks.inertia:fo}
\end{gather}
\normalsize
We may also write rule~\eqref{eq:blocks.inertia} in this form.
We assume that variable~$T$ in rules~(\ref{eq:blocks.inertia}-\ref{eq:blocks.non}) is of sort integer.
In examples, we assume that any variable that is used as argument of an arithmetic function is of this sort.
Under these assumptions, the definition of stable models of sets of sentences stated earlier also applies to (possible infinite) disjunctive logic programs with arithmetic.
Note that these assumptions are enough to showcase all examples in this paper.
For a more general characterization of how sorts are assigned to variables, we refer to~\citet{falilusc20a} and~\citet{lifschitz21a}.

%% file: here_and_there_lambda.tex
\section{Stable Models with Intensionality Statements}

The distinction between \emph{extensional} and \emph{intensional} predicates~\cite{feleli11a} describes the inherent meaning of a group of rules in a precise way; and it is similar to the distinction between input and non\nobreakdash-input atoms by~\citet{oikjan08a}. It is sometimes convenient to treat words  \emph{extensional} and \emph{intensional} as synonymous to  words \emph{input} and \emph{defined}, respectively.
Intuitively, the meaning of input predicate symbols is not fixed within the considered group of rules; these rules may constrain the interpretations of such predicate symbols but they do not ``define'' them.
To the contrary, intensional predicate symbols can be viewed as defined by the group. 
%
%
%
For instance, consider rules~(\ref{eq:blocks.inertia}-\ref{eq:blocks.non}) that we denote as~$\Pi_{\mathit{block}}$.
Let us  elaborate on the  meaning of~$\Pi_{\mathit{block}}$.
Intuitively,
$\Pi_{\mathit{block}}$ captures the behavior of the property $\on/3$  that changes when relevant movement occurs and otherwise it obeys the commonsense law of inertia (i.e., that things stay as they are unless forced to change). 
  To obtain such a reading of~$\Pi_{\mathit{block}}$ 
 predicates~$\textit{move}/3$ and~$\textit{location}/3$ should be declared as extensional.
 %
%
On the other hand, predicates~$\textit{on}/3$ and~$\textit{non}/3$ should be declared as intensional as their behavior is defined by rules in~$\Pi_{\mathit{block}}$. 
It looks as such separation of predicate symbols allows us to identify $\Pi_{\mathit{block}}$ with its intuitive meaning and yet the devil is in the details.
The reading that we have stated fails to mention the special case of the initial situation.
At the initial state, i.e., when~$T=0$, we do not assume that rules in $\Pi_{\mathit{block}}$ define $\on/3$, rather that these values are specified elsewhere.
%
%
%
%
%
Thus, we would like to declare that~$\textit{on}/3$ is extensional in the initial situation (when~$T=0$) and intensional in all other situations (when~$T \neq 0$). 
%
However, the granularity desired for this example is not currently possible because the distinction between {extensional} and {intensional} by~\citet{feleli11a}
 is made at the predicate level, disregarding the context provided by their arguments.
%
We address this issue~here.

We generalize the idea of distinguish between extensional and intensional predicate symbols by allowing the possibility to declare circumstances under which a predicate symbol is considered to be intensional (if these circumstances are not the case the predicate symbol is considered extensional). 
This is achieved by associating a \fo  formula with each predicate constant.
For instance, in  $\Pi_{\mathit{block}}$ associating predicate symbol $\textit{on}/3$  with the formula~$T \neq 0$ will result in proper treatment of its rules respecting not only their core meaning but also the corner case of the initial situation. 

%

Formally, we identify each predicate symbol~$p/n$ with an atom of the form~$p(X_1,\dotsc,X_n)$, where $X_1,\dots,X_n$ are pairwise distinct variables of appropriate sorts.
An \emph{intensionality statement}~$\pdef$ over a signature $\sigma$ is a function mapping each predicate symbol~$p/n$ in $\sigma$ to a formula~$F(X_1,\dotsc,X_n)$ such that 
\begin{itemize}
    \item $X_1\!,\!\dots\!,\!X_n$ are the only free variables in $F(X_1\!,\!\dots\!,\!X_n)$,
    \item every predicate symbol~$q/m$ in~$F$ satisfies~$\pdef(q/m) \equiv \bot$.
\end{itemize}
We abbreviate~$\pdef(p/n)(X_1,\hspace{-2pt}\text{\small$\dotsc$}\hspace{-2pt},X_n)$ as~$\pdef^p(X_1,\hspace{-2pt}\text{\small$\dotsc$}\hspace{-2pt},X_n)$ when arity~$n$ is clear from the context.
We say that a predicate symbol is~ (purely) \emph{intensional}/\emph{extensional} when it
is associated with a valid/unsatisfiable formula, respectively.
%

%
%

%
%
%

Consider, for instance, predicate symbol~$\mathit{on}/3$ and let $\pdef^\mathit{on}(B,L,T)$ be formula~$T \neq 0$.
Intuitively, this intensionality statement states that all ground atoms formed by predicate constant~$p$ with the third argument different from~$0$ are intensional; if the third argument is~$0$, then they are extensional.
The complete intensionality statement, which we call~$\pdefb$, for our motivating example follows:
\begin{gather}
\begin{aligned}
  \pdefb^{\mathit{on}}(B,L,T) \ &\text{ is } \ T \neq 0
  \\
  \pdefb^{\mathit{non}}(B,L,T) \ &\text{ is } \ \top
  \\
  \pdefb^{\mathit{move}}(B,L,T) \ &\text{ is } \ \bot
  \\
  \pdefb^{\mathit{location}}(L) \ &\text{ is } \ \bot
  \\
  \pdefb^{\prec}(X_1,X_2) \ &\text{ is }\ \bot \quad \text{for any comparison~$\prec$}
\end{aligned}
  \label{eq:blocks.intensional.specification}
\end{gather}
It states that predicate $\mathit{move}$ and all comparison symbols are purely extensional, $\mathit{non}$ is purely intensional and~$\mathit{on}$ is intensional under all circumstances except the initial situation.

Let $\sigma$ be a signature and $\pdef$ be an intensionality statement over~$\sigma$. 
By~$\EM{\pdef}$, we denote the set consisting of a sentence of the form
\begin{gather*}
\forall\boldX \left( \neg \pdef^p(\boldX) \to p(\boldX) \vee \neg p(\boldX) \right)
\end{gather*}
for every predicate symbol~$p/n$ in~$\sigma$, where~$\boldX$ is a tuple of variables of the appropriate length and sort (in the sequel we adopt this convention and use $\boldX$ to denote  tuples of variables).
For a theory~$\Gamma$ over~$\sigma$, 
we say that an interpretation~$I$ is $\pdef$\nobreakdash-\emph{stable} if it is a stable model of~$\Gamma \cup \EM{\pdef}$.
Note that, if every predicate symbol is intensional in~$\pdef$, then the $\pdef$\nobreakdash-stable models coincide with the stable models.

For example,
let~$F$ be the rule
\begin{align}
p(X,1) &\to p(X,2)  \label{eq:pr1}
\end{align}
and 
let~$\pdef^p(X_1,X_2)$ be formula~$X_2 = 2$ (we assume both arguments of $p/2$ be of sort integer).
The intensionality statement formula $\EM{\pdef}$ is the universal closure of
\begin{align*}
\neg(X_2 = 2) &\to p(X_1,X_2) \vee \neg p(X_1,X_2).
\end{align*}
The following sets of ground atoms correspond to the four $\pdef$\nobreakdash-stable models of~$F$ with domain~$\{1,2\}$
\begin{gather*}
\emptyset \quad \quad
\{p(1,1),\,p(1,2)\} \quad \quad
\{p(2,1),\,p(2,2)\} \\
\{p(1,1),\,p(2,1),\,p(1,2),\,p(2,2)\},	  
\end{gather*}
where we  list the atoms that are true in these models (as customary in logic programming).
As a more elaborated example, we can see that an interpretation~$I$ that satisfies
\begin{align*}
  \mathit{location}^I &= \{ l_1,l_2 \}
  \hspace*{17pt}
  \mathit{move}^I = \{ (b,l_2,0 ) \}
  \\
  \mathit{on}^I &= \{ (b,l_1,0 ) \} \cup \{ (b,l_2,t' ) \mid t' > 0 \}
  \\
  \mathit{non}^I &= \{ (b,l_2,0 ) \} \cup \{ (b,l_1,t' ) \mid t' > 0 \}
\end{align*}
is a $\pdefb$\nobreakdash-stable model of~$\Pi_{\mathit{block}}$.

\section{Strong Equivalence with Intensionality Statements}

Informally, two programs are strongly equivalent if one can replace the other in any context without changing the stable models of the program~\cite{lipeva01a,lipeva07a}.
Let~$\Gamma_1$ and~$\Gamma_2$ be theories over the same signature~$\sigma$.
We say that sets~$\Gamma_1$ and~$\Gamma_2$ are \emph{strongly equivalent with respect to
an intensionality statement~$\pdef$} (or, simply $\pdef$\nobreakdash-\emph{strongly equivalent}) if~$\Gamma_1 \cup \Delta$ and~$\Gamma_2 \cup \Delta$ have the same~$\pdef$\nobreakdash-stable models for any theory~$\Delta$ over~$\sigma$.

By definition of~$\pdef$\nobreakdash-stable model, we immediately obtain the following sufficient condition for $\pdef$\nobreakdash-strong equivalence.
\begin{proposition}\label{prop:strong.equivalence}
If~$\Gamma_1\cup \EM{\pdef}$ and~$\Gamma_2\cup \EM{\pdef}$ have the same HT\nobreakdash-models, then~$\Gamma_1$ and~$\Gamma_2$ are $\pdef$\nobreakdash-strongly equivalent.
\end{proposition}
Proposition~\ref{prop:strong.equivalence} allows us to conclude that we can replace, without changing the~$\pdefb$\nobreakdash-stable models, rule~\eqref{eq:blocks.inertia} in~$\Pi_{\mathit{block}}$ by rules
\small
\begin{align}
\mathit{on}(B,L,T+1) &\leftarrow \mathit{on}(B,L,T),
\label{eq:blocks.inertia.leq}
\\
&\phantom{\leftarrow}\hspace{5pt} not\
\mathit{non}(B,L,T+1),\ T < t
\notag
\\
\mathit{on}(B,L,T+1) &\leftarrow \mathit{on}(B,L,T),
\label{eq:blocks.inertia.geq}
\\
&\phantom{\leftarrow}\hspace{5pt} not\
\mathit{non}(B,L,T+1),\
T \geq t
\notag
\end{align}
\normalsize
where~$t$ is any positive integer.
Note that~\eqref{eq:blocks.inertia.leq} and~\eqref{eq:blocks.inertia.geq} are obtained from rule~\eqref{eq:blocks.inertia} by respectively adding~$T < t$ and~$T \geq t$ to its body.
In fact~\eqref{eq:blocks.inertia} is~$\pdef$\nobreakdash-strong equivalent to to~$\{\eqref{eq:blocks.inertia.leq} ,\eqref{eq:blocks.inertia.geq} \}$ for every~$\pdef$ where comparison symbols are extensional.
%
%
Analogous rewritings can be performed on the remaining rules of~$\Pi_{\mathit{block}}$.
By~$\Pi_{\mathit{block}}^<$ we denote the program obtained from~$\Pi_{\mathit{block}}$ by adding~${T < t}$ to the body of rules~\eqref{eq:blocks.inertia} and~\eqref{eq:blocks.action}, and adding~${T \leq t}$ to the body of rule~\eqref{eq:blocks.non}.
By~$\Pi_{\mathit{block}}^>$ we denote the program obtained from~$\Pi_{\mathit{block}}$ by adding~${T \geq t}$ to the body of rules~\eqref{eq:blocks.inertia} and~\eqref{eq:blocks.action}, and adding~${T > t}$ to the body of rule~\eqref{eq:blocks.non}.
Programs~$\Pi_{\mathit{block}}$ and~$\Pi_{\mathit{block}}^< \cup \Pi_{\mathit{block}}^>$ are~$\pdefb$\nobreakdash-strongly equivalent.
We use program~$\Pi_{\mathit{block}}^< \cup \Pi_{\mathit{block}}^>$ to showcase our Splitting method in the following section. In particular, the proposed method allows us to split this program into its two subcomponents~$\Pi_{\mathit{block}}^<$ and $\Pi_{\mathit{block}}^>$.

%% file: splitting_logic_programs.tex
\section{Splitting Theorem with Intensionality Statements}
\label{sec:splitting}

To make the key results of this work comprehensible, we describe our Splitting method using two settings. 
%
First, we introduce the core ideas of the method in a restricted context, where considered theories are  composed of disjunctive rules.
Then, we generalize the method to arbitrary theories.

\subsection{Splitting Disjunctive Programs}

The statement of the proposed Splitting Theorem refers to the concept of the predicate dependency graph defined below.
This definition, requires the notion of a \emph{partition} of an intensionality statement.
Formally, given two intensionality statement, $\pdef_1$ and~$\pdef_2$, over a signature~$\sigma$ by~$\pdef_1 \sqcup \pdef_2$  we denote the intensionality statement defined as
\small
\begin{gather*}
(\pdef_1 \sqcup \pdef_2)^p(X_1,\text{\footnotesize\dots},X_n) 
= \pdef_1^p(X_1,\text{\footnotesize\dots},X_n) \vee \pdef_2^p(X_1,\text{\footnotesize\dots},X_n).
\end{gather*}
for every predicate symbol~$p$ in the signature.
\normalsize
By~${\pdef_1 \sqcap \pdef_2}$, we denote the intensionality statements defined as
\small
\begin{gather*}
(\pdef_1 \sqcap \pdef_2)^p(X_1,\text{\footnotesize\dots},X_n) 
= \pdef_1^p(X_1,\text{\footnotesize\dots},X_n) \wedge \pdef_2^p(X_1,\dotsc,X_n).
\end{gather*}
\normalsize
for every predicate symbol~$p$.
Intensionality statements $\pdef_1$ and~$\pdef_2$ are \emph{equivalent} if the universal closure of formula
\begin{gather*}
  \pdef_1^p(X_1,\dotsc,X_n) \leftrightarrow \pdef_2^p(X_1,\dotsc,X_n)
\end{gather*}
is logically valid for every predicate symbol~$p$.
If $\pdef_1$ and~$\pdef_2$ are equivalent, then we write~${\pdef_1 \equiv \pdef_2}$.
By~$\pdef_\top$ and $\pdef_\bot$ we denote the intensionality statements where all predicate symbols are intensional and extensional, respectively.
%
%
Intensionality statements $\pdef_1$ and~$\pdef_2$ are called \emph{disjoint} if~${\pdef_1 \sqcap \pdef_2 \equiv \pdef_\bot}$.
A set of intensionality statements~$\{\pdef_1,\dotsc,\pdef_k \}$ is called a \emph{partition} of some intensionality statement~$\pdef$ if~${\pdef_1 \sqcup \dotsc \sqcup \pdef_k \equiv \pdef}$ and~${\pdef_i \sqcap \pdef_j \equiv \pdef_\bot}$ for all~$i \neq j$.
For instance, we can see that~$\{ \pdefb_1, \pdefb_2 \}$ is a partition of intensionality statement~$\pdefb$,
where~$\pdefb_1$ and~$\pdefb_2$ are obtained from~$\pdefb$ by modifying the formulas associated with predicate constants~$\mathit{on}$ and~$\mathit{non}$ as~follows:
\begin{gather}
\begin{aligned}
  \pdefb_1^{on}(B,L,T) \ &\text{ is } \ T \neq 0 \wedge T \leq t
  \\
  \pdefb_1^{non}(B,L,T) \ &\text{ is } \ T \leq t
  \\
  \pdefb_2^{on}(B,L,T) \ &\text{ is } \ T > t
  \\
  \pdefb_2^{non}(B,L,T) \ &\text{ is } \ T > t.
\end{aligned}
  \label{eq:blocks.intensional.specification.splitted}
\end{gather}
An occurrence of a predicate constant, or any
other subexpression, in a formula is called \emph{negated} if it occurs in the scope of negation (i.e in the antecedent of any implication of the form~$F \to \bot$); \emph{nonnegated} otherwise.
%
%
As an example, predicate constant $\mathit{on}$ occurs nonnegated in the antecedent of~\eqref{eq:blocks.inertia:fo}, while~$\mathit{non}$ does not.
Given a program~$\Pi$ (understood as a set of sentences) and a partition~$\pdefs$ of some intensionality statement~$\pdef$, the \emph{(directed) graph of dependencies with respect to~$\pdefs$},
denoted~$\G_{\pdefs}(\Pi)$, is  defined as follows:
\begin{itemize}
  \item Its vertices are pairs~$(p,\pdef_i)$ such that~$p$ is a predicate symbol occurring in~$\Pi$, $\pdef_i \in \pdefs$ is an intensionality statement and~$\exists \boldX\, \pdef_i^p(\boldX)$ is satisfiable.
  \item It has an edge from~$(p,\pdef_i)$ to~$(q,\pdef_j)$ when for some rule~$B \to H$ of~$\Pi$,
  \begin{itemize}
  	\item 
  $p(\boldt)$ occurs in~$H$, and 
  \item 
  there is a nonnegated occurrence of~$q(\boldr)$ in~$B$  and
  \item
  the following sentence is satisfiable
  \begin{gather}
    \exists \boldX\, \big( B \wedge p(\boldt) \wedge \pdef_i^p(\boldt) \wedge \pdef_j^q(\boldr) \big),
    \label{eq:edge.sentence}
  \end{gather}
  where~$\boldX$ are the free variables in~$B \to H$.
\end{itemize}
\end{itemize}
We say that a partition~$\pdefs = \{ \pdef_1, \dotsc, \pdef_k \}$  is \emph{separable} (on~$\G_{\pdefs}(\Pi)$) when
\begin{itemize}
\item[] every infinite walk~$v_1,v_2,\dotsc$ of~$\G_{\pdefs}(\Pi)$ visits at most one~$\pdef_i$ infinitely many times, that is, 
there is some~$i \in \{1, \dotsc, k \}$ s.t
$\{ l \mid v_l\!=\! (p,\pdef_j) \}$ is finite for all~${j \neq i}$.
\end{itemize}
Note that, if (the signature of)~$\Pi$ is finite, every possible infinite walk must contain a cycle and, thus, this condition is equivalent to saying that every cycle is confined to one component of the partition.
%
In general, checking cycles is not enough for infinite programs, even when these programs contain no variables~\cite{harlif16a}.

Continuing with our running example, let~$\Pi_\mathit{block}' = \Pi_\mathit{block}^< \cup \Pi_\mathit{block}^>$ and $\pdefs_\mathit{block}=\{\pdefb_1,\pdefb_2\}$ be partition of~$\pdefb$.
Then, graph~$\G_{\pdefs_\mathit{block}}(\Pi_\mathit{block}')$
consists of vertices
$$
(on,\pdefb_1),\quad (on,\pdefb_2),\quad (non,\pdefb_1),\quad (non,\pdefb_2)
$$
and contains five edges: one leading from~$(on,\pdefb_1)$ to itself, 
another one leading from~$(on,\pdefb_2)$ to itself,
one leading from~$(on,\pdefb_2)$ to~$(on,\pdefb_1)$\,---\,induced by rule~\eqref{eq:blocks.inertia.geq}\,---\,
and two leading from~$(non,\pdefb_i)$ to~$(on,\pdefb_i)$ with~${1 \leq i \leq 2}$.
Note that there is no edge from~$(on,\pdefb_1)$ to~$(on,\pdefb_2)$.
For instance, we can see that rules~\eqref{eq:blocks.inertia.leq} and~\eqref{eq:blocks.inertia.geq} do not induce such edge because the sentence of the form~\eqref{eq:edge.sentence} corresponding to these rules contain conjuncts~${T+1 \leq t}$
and~\hbox{$T > t$}, which make them unsatisfiable.
Conjunct~${T+1 \leq t}$ is obtained from~$\pdefb_1^\mathit{on}(B,L,T+1)$ corresponding to the head occurrence of~$\mathit{on}/3$.
The other conjunct is obtained from~$\pdefb_2^\mathit{on}(B,L,T)$ corresponding to its body occurrence.
Therefore, partition~$\pdefs_\mathit{block}$ is separable.

A program~$\Pi$ is \emph{negative} on some intensionality statement~$\pdef$ if, for every rule~$B \to H$ of~$\Pi$ and every atom~$p(\boldt)$ occurring in~$H$, the following sentence  is unsatisfiable
$${\exists \boldX\, (B \wedge p(\boldt) \wedge \pdef^p(\boldt))},$$ where~$\boldX$ are the free variables in~${B \to H}$.
For instance, programs~$\Pi_\mathit{block}^<$ and~$\Pi_\mathit{block}^>$ are negative on~$\pdefb_2$ and~$\pdefb_1$, respectively.
We  provide a part of an argument for the claim that 
$\Pi_\mathit{block}^<$ is negative on~$\pdefb_2$. In particular, we argue that rule~\eqref{eq:blocks.inertia.leq} is negative on~$\pdefb_2$.
This rule has~$\mathit{on}(B,L,T+1)$ in its head and~$T < t$ in its body.
It is sufficient to show that a formula of the form
\begin{align*}
\exists BLT\, (F \wedge T < t \wedge  \pdefb^{on}_2(B,L,T+1)), 
\end{align*}
is unsatisfiable.
Recall that~$\pdefb_2^{on}(B,L,T+1)$ is ${T+1 > t}$.
%
%
%
\begin{theorem}[Splitting disjunctive programs]\label{thm:splitting}
Let~$\Pi \!=\! \Pi_1 \!\cup\! \text{\footnotesize\dots} \!\cup \Pi_n$\! be a disjunctive program
and $\pdefs \!=\! \{ \pdef_1, \text{\footnotesize\dots} ,\pdef_n \}$ be a partition of~$\pdef$ such that
\begin{itemize}
  \item $\pdefs$ is separable on~$\G_{\pdefs}(\Pi)$; and
  \item each~$\Pi_i$ is negative on~$\pdef_j$ for all~$j \neq i$.
\end{itemize}
Then, for any interpretation~$I$,
the following two statements are equivalent
\begin{itemize}
    \item $I$ is a $\pdef$\nobreakdash-stable model of~$\Pi$, and
    \item $I$ is a $\pdef_i$-stable model of~$\Pi_i$ for all~$1\leq i\leq n$.
\end{itemize}
\end{theorem}
Recall that partition~$\pdefs_\mathit{block}$ is separable on $\G_{\pdefs_\mathit{block}}(\Pi_\mathit{block}')$ and that programs~$\Pi_\mathit{block}^<$ and~$\Pi_\mathit{block}^>$ are negative on~$\pdefb_2$ and~$\pdefb_1$, respectively.
Therefore, the theorem on Splitting disjunctive programs allows us to conclude that the $\pdefb$\nobreakdash-stable models of~$\Pi_\mathit{block}'$ are those interpretations that are at the same time~$\pdefb_1$\nobreakdash-stable models of~$\Pi_\mathit{block}^<$
and~$\pdefb_2$\nobreakdash-stable models of~$\Pi_\mathit{block}^>$.
As a final note, recall that~$\Pi_\mathit{block}$ and~$\Pi_\mathit{block}'$ are $\pdefb$-strongly equivalent and, thus, they have the same $\pdefb$\nobreakdash-stable models.
Hence, the statement that we can split program~$\Pi_\mathit{block}$ into programs~$\Pi_\mathit{block}^<$ and~$\Pi_\mathit{block}^>$ is a consequence of these two facts stemming from presented
Theorem~\ref{thm:splitting} and
 Proposition~\ref{prop:strong.equivalence}, respectively.

%% file: splitting_theories.tex
\subsection{Splitting Arbitrary Theories}

We now generalize Theorem~\ref{thm:splitting} on Splitting disjunctive programs to the case of arbitrary sets of sentences. We start by motivating this generalization and then proceed to  its formalization.
Complex formulas not fitting into the syntax of disjunctive rules naturally appear as a result of translating some common constructs of logic programs into first\nobreakdash-order formulas.
These constructs include basic arithmetic operations such as division, which are part of the ASP\nobreakdash-Core\nobreakdash-2~\cite{aspcore2}, and more advanced features such as intervals, or conditional literals.
Intervals and conditional literals are part of the language of the solver~\texttt{clingo}~\cite{gehakalisc15a} and are commonly used in practice.
%
For instance, the  following rule contains a conditional literal \hbox{$ holds(B) : body(R,B)$} in its body:
\begin{align}
    holds(H) &\leftarrow head(R,H),\, holds(B) : body(R,B).    \label{eq:metaencoding.rule.positive}
\end{align}
\citet{hanlie22a} showed that
this rule can be understood as an abbreviation for the following first\nobreakdash-order sentence containing a nested universal quantifier and a nested implication
\begin{multline}
    \forall RH \big( head(R,H) \,\wedge\\ \forall B  \big(body(R,B) \to holds(B)\big) \to holds(H) \big).
    \label{eq:metaencoding.positive}
\end{multline}
Rules of with condition literals are common in meta-programming~\cite{karoscwa21a}, where reification of constructs are utilized to build ASP-based reasoning engines that {\em may} go beyond ASP paradigm itself.
\citet{karoscwa21a} presented multiple examples of meta-programming.
Here, we use rule~\eqref{eq:metaencoding.rule.positive} to showcase a simple use of this technique.
Indeed, it can be used to express the meaning of a \emph{definite} rule -- a disjunctive rule whose $\mathit{Head}$ consists of single atomic formula and $\mathit{Body}$ contains no occurrences of negation $not$ --
 encoded as a set of facts.
For example, definite program 
\begin{equation}
    \begin{array}{l}
    a \leftarrow b\\
    b \leftarrow c
    \end{array}
    \label{eq:basic-pr}
\end{equation}
can be encoded by the following facts complemented with rule~\eqref{eq:metaencoding.rule.positive} 
\begin{gather}
    head(r1,a) \quad body(r1,b) 
    \label{eq:metaencoding.r1}
    \\
    head(r2,b) \quad body(r2,c).
    \label{eq:metaencoding.r2}
\end{gather}
Clearly, we can split the program listed in~\eqref{eq:basic-pr} into two subprograms, each consisting of a single rule.
We may expect the possibility of splitting its respective meta encoding as well.
Consider sentences
\begin{align}
    \begin{aligned}
        \forall X &\big( \mathit{head}(r1,X) \,\wedge\\
            &\forall W \big(\mathit{body}(r1,W) \to \mathit{holds}(W)\big) \to \mathit{holds}(X) \big)        
    \end{aligned}
\label{eq:metaencoding.positive.r1}
\\
\begin{aligned}
    \forall X &\big( \mathit{head}(r2,X) \,\wedge\\
    &\forall W \big(\mathit{body}(r2,W) \to \mathit{holds}(W)\big) \to \mathit{holds}(X) \big)    
\end{aligned}
\label{eq:metaencoding.positive.r2}
\end{align}
corresponding to a ``partially instantiated'' portion of the meta-encoding of program~\eqref{eq:basic-pr}.
%
Ideally,  the program corresponding to sentences~\mbox{(\ref{eq:metaencoding.r1}-\ref{eq:metaencoding.positive.r2})} 
would be ``identified'' with
 subprograms~\mbox{(\ref{eq:metaencoding.r1},\ref{eq:metaencoding.positive.r1})} and~(\ref{eq:metaencoding.r2},\ref{eq:metaencoding.positive.r2}) by splitting.
Yet, Theorem~\ref{thm:splitting} cannot support such a claim as~\eqref{eq:metaencoding.positive.r1} and~\eqref{eq:metaencoding.positive.r2} are not disjunctive rules.
This claim is also not supported by other Splitting Theorems in the literature~\cite{felelipa09a,harlif16a} due to the ``positive nonnegated'' dependency induced by these rules in the predicate~$\mathit{hold}/1$ (we define the concept of positive nonnegated dependency below).

\subsubsection{Formalization.}

The first thing to note is that one direction of the Splitting Theorem always holds without the need to inspect the dependency graph.

\begin{proposition}\label{prop:splitting}
Let~$\Gamma = \Gamma_1 \cup \dots \cup \Gamma_n$ be some theory and let
$\pdefs = \{ \pdef_1, \dots ,\pdef_n \}$ be a partition of intensionality statement~$\pdef$.
If~$I$ is a $\pdef$\nobreakdash-stable model of~$\Gamma$, then~$I$ is a $\pdef_i$-stable model of~$\Gamma_i$ for all~$1\leq i\leq n$.
\end{proposition}

For the other direction, we have to generalize the notions of separability and  of being negative to the case of arbitrary sentences.
Similar to the Splitting Theorem by~\citet{felelipa09a}, this generalization relies on the notions of~\emph{strictly positive}, \emph{positive nonnegated}, and \emph{negative nonnegated} occurrence of an expression.
An occurrence of an expression is called \emph{positive} if the number of implications containing that occurrence in the antecedent is even;
and \emph{strictly positive} if that number is~$0$.
It is called \emph{negative} if that number is odd.
As in the case of logic programs, an occurrence of an expressions is called \emph{negated} if it belongs to a subformula of the form~$\neg F$ (that is, $F \to \bot$); and \emph{nonnegated} otherwise.
A \emph{rule} is a strictly positive occurrence of a formula of the form~${B \to H}$.
Therefore, the rules of a set of disjunctive rules are exactly all its disjunctive rules, but this definition also covers arbitrary nested  rules. For instance, 
sentence~${(c \to (a \to b)) \vee d)}$ contains two occurrences of rules, namely, ${a \to b}$ and~\hbox{$c \to (a \to b)$}.

To extend Theorem~\ref{thm:splitting} to arbitrary formulas, we use these notions to make the construction of a counterpart of sentence~\eqref{eq:edge.sentence} recursive over the formula.
In addition, we ought to incorporate in this construction a \emph{context} that carries information about the rest of the program.
To observe  the need for this context, note that whether rules~\eqref{eq:metaencoding.positive.r1} and~\eqref{eq:metaencoding.positive.r2} can be separated into differ subprograms depends on the extension of~$\mathit{head}/2$ and~$\mathit{body}/2$.
For instance, adding the fact~$\mathit{body}(r2,a)$ to our program creates a dependency that cannot be broken.
In fact, this new program has the same meaning as a non-splittable program consisting of rules \hbox{$a \leftarrow b$} and~$b \leftarrow c, a$.

Given a formula~$F$ with free variables~$\boldX$ and a theory~$\Psi$, by~$F^\Psi$ we denote formula~$F$ itself if~${\Psi \cup \{ \exists \boldX\,F \}}$ is satisfiable; and~$\bot$ otherwise.
Let now~$p(\boldt)$ be a strictly positive (resp. positive nonnegated or negative nonnegated) ``distinguished'' occurrence of predicate~$p$ in~$F$ and~$\boldY$ a list of variables not occurring in~$F$ of the same length and sorts as~$\boldt$.
We recursively build formula~$\Pos{\Psi}{F}$ (resp. $\Pnn{\Psi}{F}$ and~$\Nnn{\Psi}{F}$) for this occurrence as described next;
the construction of these three formulas only differs in the case of implication, so we use the metavariable~$\Fun{\Psi}{F} $ for the common cases:
\begin{itemize}
    \item If~$\Psi \cup \{ \exists \boldX\,F \}$ is unsatisfiable, then~$\Fun{\Psi}{F} = \bot$.
    \item If~$\Psi \cup \{ \exists \boldX\,F \}$ is satisfiable and the distinguished occurrence~$p(\boldt)$ does not occur in~$F$, then~$\Fun{\Psi}{F} = F$.
\end{itemize}
Otherwise, 
\begin{itemize}
\item $\Fun{\Psi}{p(\boldt)} = p(\boldt) \wedge \boldY = \boldt$; 

\item $\Fun{\Psi}{F_1 \wedge F_2} = \Fun{\Psi}{F_1} \wedge
    \Fun{\Psi}{F_2}$;

\item $\Fun{\Psi}{F_1 \vee F_2} =  \Fun{\Psi}{F_i}$ with $F_i$ containing the  occurrence ($i$ is $1$ or $2$);

\item $\Fun{\Psi}{\forall X F} = \Fun{\Psi}{\exists X F} = \exists X \big( \Fun{\Psi}{F}\big)$;

\item $\Pos{\Psi}{F_1 \to F_2} = F_1^\Psi \wedge \Pos{\Psi}{F_2}$;

\item $\Pnn{\Psi}{F_1 \to F_2} = \\\hspace*{10pt}\begin{cases}
    \Nnn{\Psi}{F_1} &\text{if $F_1$ contains the occurrence}
    \\
    F_1^\Psi \wedge \Pnn{\Psi}{F_2} &\text{otherwise}
    \end{cases}$;

\item $\Nnn{\Psi}{F_1 \to F_2} =\\\hspace*{10pt}\begin{cases}
    \Pnn{\Psi}{F_1} &\text{if $F_1$ contains the occurrence}
    \\
    F_1^\Psi \wedge \Nnn{\Psi}{F_2} &\text{otherwise}
    \end{cases}$.
\end{itemize}
Note that~$\Pnn{\Psi}{F}$ and~$\Nnn{\Psi}{F}$ are mutually recursive due to the case of implication.
%
Examples illustrating the construction of these formulas follow the next definition.


Given theories~$\Gamma$ and~$\Psi$ and a partition~$\pdefs$ of some intensionality statement~$\pdef$, the \emph{(directed) graph of positive dependencies with respect to~$\pdefs$ and under context~$\Psi$},
denoted~$\G_{\pdefs,\Psi}(\Gamma)$, is  defined as follows:
\begin{itemize}
    \item Its vertices are pairs~$(p,\pdef_i)$ such that~$p$ is a predicate symbol, $\pdef_i \in \pdefs$ is an intensionality statement and theory~$\Psi \cup \{ \exists \boldX\, \pdef_i^p(\boldX) \}$ is satisfiable.

    \item It has an edge from~$(p,\pdef_i)$ to~$(q,\pdef_j)$ when for some rule~$B \to H$ of~$\Gamma$, the following conditions hold
    \begin{itemize}
    	\item 
    	    	there is a strictly positive occurrence~$p(\boldt)$  in~$H$,
    	\item 
  there is a positive nonnegated occurrence of~$q(\boldr)$ in~$B$, 
  \item the theory below is satisfiable, 
    \small
    \begin{gather}
    \hspace*{-25pt}
    \Psi \cup \big\{ 
    \exists \boldX\boldY\boldZ \big( \Pnn{\Psi}{B} \wedge \Pos{\Psi}{H} \wedge \pdef_j^q(\boldY) \wedge \pdef_i^p(\boldZ) \big) \big\},        \label{eq:edge.sentence.arbitrary.formulas}
    \end{gather}
    \normalsize
     where~$\boldX$ are the free variables in~${B \to H}$,
    and $\boldY$ and $\boldZ$ respectively are the free variables
    in formulas~$\Pnn{\Psi}{B}$ and~$\Pos{\Psi}{H}$ that are not in~$\boldX$;  in the construction of     
    $\Pos{\Psi}{H}$ and $\Pnn{\Psi}{B}$ we consider occurrences   $p(\boldt)$ in $H$ and $q(\boldr)$ in $B$, respectively.
\end{itemize}

\end{itemize}
Take sentence~${F_1 = q \vee (r \wedge p)}$ and rule 
\hbox{$F_1 \to p$}.
Note that~$p$ occurs positively nonnegated in~$F_1$.
 We consider this occurrence $p$ when constructing formulas $\Pnn{(\cdot)}{F_1}$ below.
Assume that~$p$ is intensional in~$\pdef$ (i.e., $\pdef^p\equiv\top$) and that~$\pdefs$ is a partition of~$\pdef$.
On the one hand, if we consider the empty context, then $\Pnn{\emptyset}{F_1}$ is~$r \wedge p$ and the dependency graph of~$F_1 \to p$ with respect to $\pdefs$ and the empty context
 contains a reflexive edge on vertex~$(p,\pdef)$.
On the other hand, if we take context~$\Psi_1 = \{\neg r\}$, then $\Pnn{\Psi_1}{F_1}$ is~$\bot$ and the dependency graph with respect to $\pdefs$ and $\Psi_1$ is empty.
This example shows that, even on propositional formulas, the use of context leads to less dependencies than previous approaches such as~\cite{felelipa09a,harlif16a}.
This leads to the fact that the Splitting Theorem introduced in the sequel is applicable to more theories as it relies on this new  notion of dependencies.
This difference may appear even using the empty context. 
Take a theory consisting of formula  $\big(q \vee (\bot \wedge p)\big)\to p$. It turns out that  $p$ never contributes any edge in our approach 
under any context. For the same theory,
$p$ always forms a dependency in the earlier papers.
%
%
Consider now a small variation. Let~$p$ be a unary predicate and let~${F_2 = q \vee \exists X \big(r \wedge p(X)\big)}$. Note that
$\Pnn{\emptyset}{F_2}=\exists X (r \wedge p(X) \wedge Y = X)$ and $\Pnn{\Psi_1}{F_2} = \bot$, where we construct these formulas for occurrence~$p(X)$.
Under the empty context, analogously to the previous example, the dependency graph of~$F_2 \to p(1)$ contains a reflexive edge on vertex~$(p,\pdef)$. Under context~$\Psi_1$, the dependency graph is empty.
Construction of~$\Pos{\Psi}{F}$, $\Pnn{\Psi}{F}$ and~$\Nnn{\Psi}{F}$
reference a ``distinguished'' occurrence, as there may be multiple occurrences of the same atomic formula.
For instance, there are two occurrences of formula~$p(X)$ 
in sentence~$F_3$ defined as ${\exists X (X \!=\! a \!\wedge\! p(X)) \!\vee\! (X \!=\! b \!\wedge\! p(X))}$; for this sentence
$\Pnn{\emptyset}{F_3}$ is~${\exists X  (X \!=\! a \!\wedge\! p(X) \!\wedge\! Y \!=\! X)}$, when the first occurrence of $p(X)$ is considered;
and~${\exists X (X \!=\! b \!\wedge\! p(X) \!\wedge\! Y \!=\! X)}$, when the second one is considered.
Take~$\Psi_2$  to be the theory containing sentence~$\neg p(b)$ only, then~$\Pnn{\Psi_2}{F_3}$ remains unmodified for the first occurrence of $p$, but becomes~${\bot}$ for the second one.
As a consequence, in a context where~$p(b)$ is false, the first occurrence may generate an edge in the dependency graph, while the second occurrence does not generate any.
Note that in all considered examples so far, $\Pos{(\cdot)}{\cdot}$ is the same as~$\Pnn{(\cdot)}{\cdot}$ due to the lack of implications.
%

%
We now explore the meta encoding example introduced earlier.
Let~$B$ denote the antecedent of rule~\eqref{eq:metaencoding.positive.r1} and $H$ denote its consequent $\mathit{holds}(X)$. 
%
Under the empty context, $\Pnn{\emptyset}{B}$ for the only occurrence of $\mathit{holds}/1$ in $B$ is
$${head(r1,X) \wedge \exists W \big(\mathit{body}(r1,W) \wedge \mathit{holds}(W) \wedge Y = W\big)};$$
and~$\Pos{\emptyset}{H}$
for the only occurrence of $\mathit{holds}/1$ in $H$ 
 is~$\mathit{holds}(X) \wedge X = Z$.
Let~$\pdefc_1$ and~$\pdefc_2$ be intensionality statements where~$\mathit{head}$ and~$\mathit{body}$ are extensional and 
\begin{gather*}
\begin{aligned}
    \pdefc_1^{\mathit{holds}}(X_1) &\text{ is } X_1 = a
\end{aligned}
\quad
\begin{aligned}
    \pdefc_2^{\mathit{holds}}(X_1) &\text{ is } X_1 = b.
\end{aligned}
\end{gather*}
Then,  the singleton theory with the existential closure of
\begin{gather}  
    \Pnn{\emptyset}{B} \wedge \Pos{\emptyset}{H} \wedge \pdefc_j^{\mathit{holds}}(Y) \wedge \pdefc_i^{\mathit{holds}}(Z)
    \label{eq:metaprogramming.edges}
\end{gather}
is satisfiable for any~$i,j \in \{1,2\}$.
For a theory containing rule \eqref{eq:metaencoding.positive.r1}, the graph of positive dependencies  with respect to $\{\pdefc_1, \pdefc_2\}$ and under the empty context contains an edge from vertex~$(\mathit{holds},\pdefc_1)$ to vertex~$(\mathit{holds},\pdefc_2)$, and vice\nobreakdash-versa; together with reflexive edges in both vertices.
%
%
Let~$\Psi_3$ be the theory consisting of the universal closures of formulas
\begin{align*}
\mathit{head}(r1,X) &\leftrightarrow X = a \hspace{25pt}    
&\mathit{head}(r2,X) &\leftrightarrow X = b
\\
\mathit{body}(r1,X) &\leftrightarrow X = b  \hspace{25pt} 
&\mathit{body}(r2,X) &\leftrightarrow X = c.
\end{align*}
Then,
$\Pos{\Psi_3}{H}= \Pos{\emptyset}{H}$ and
$\Pnn{\Psi_3}{B}= \Pnn{\emptyset}{B}$.
The union of $\Psi_3$ and
the singleton theory consisting of the existential closure of formula~\eqref{eq:metaprogramming.edges} is satisfiable only with~${i = 1}$ and~${j = 2}$.
Therefore, under this context, the graph of positive dependencies from the previous example contains an edge from~$(\mathit{holds},\pdefc_1)$ to~$(\mathit{holds},\pdefc_2)$ but not vice\nobreakdash-versa.
There are no reflexive edges either.
In fact, under this context the positive dependency graph  of~\mbox{(\ref{eq:metaencoding.positive.r1}-\ref{eq:metaencoding.positive.r2})} only contains this edge.


Let us now generalize the notion of being a negative occurrence to the case of arbitrary theories.
We say that a theory~$\Gamma$ is $\Psi$\nobreakdash-\emph{negative} on some intensionality statement~$\pdef$ if, for every rule~$B \to H$ of~$\Gamma$ and every strictly positive occurrence~$p(\boldt)$ in~$H$, the following theory is unsatisfiable
~${\Psi \cup \{\exists \boldX \boldY\, (B \wedge \Pos{\Psi}{H} \wedge \pdef^p(\boldY))\}}$, where~$\boldX$ consists of the free variables in~${B \to H}$;~$\boldY$ consists of the free variables in~$\Pos{\Psi}{H}$ that do not occur in~$\boldX$; and we consider occurrence~$p(\boldt)$ in~$H$ in the construction of~$\Pos{\Psi}{H}$.
%
%
We also say that theory~$\Psi$ is an $\pdef$\nobreakdash-\emph{approximator} of~$\Gamma$ if all the $\pdef$-stable models of~$\Gamma$ are models of~$\Psi$.

\begin{theorem}[Splitting Theorem]\label{thm:theory.splitting}
Let~$\Gamma \!=\! \Gamma_1 \!\cup\! \text{\footnotesize\dots} \!\cup \Gamma_n$, 
$\pdef$ be an intensionality statement,
and~$\Psi$ be an~$\pdef$\nobreakdash-approximator of~$\Gamma$.
Let $\pdefs \!=\! \{ \pdef_1, \text{\footnotesize\dots} ,\pdef_n \}$ be a partition of~$\pdef$ such that
\begin{itemize}
    \item $\pdefs$ is separable on~$\G_{\pdefs,\Psi}(\Gamma)$; and
    \item each~$\Gamma_i$ is~$\Psi$-negative on~$\pdef_j$   for all~$j \neq i$.
\end{itemize}
Then,
the following two statements are equivalent
\begin{itemize}
    \item $I$ is a $\pdef$\nobreakdash-stable model of~$\Pi$, and
    \item $I$ is a model of~$\Psi$ and a~$\pdef_i$-stable model of~$\Pi_i$ for all \hbox{$i \in \{1, \dotsc, n \}$}.
\end{itemize}
\end{theorem}

Continuing with our meta encoding example, in addition to~$\pdefc_1$  and~$\pdefc_2$, we define ~$\pdefc_3$ to be the intensionality statement  where~$\mathit{head}/2$, $\mathit{body}/2$ are intensional and~$\mathit{holds}/1$ is extensional. Let~${\pdefc}$ be $\pdefc_1 \sqcup \pdefc_2 \sqcup \pdefc_3$. It is easy to see that $\{ \pdefc_1, \pdefc_2, \pdefc_3 \}$ is a partition of~${\pdefc}$.
Let~$\Gamma_1$ and~$\Gamma_2$ be the singleton theories consisting of~\eqref{eq:metaencoding.positive.r1} and~\eqref{eq:metaencoding.positive.r2}, respectively;
let~$\Gamma_3$ consist of facts in~(\ref{eq:metaencoding.r1}-\ref{eq:metaencoding.r2}), and let~$\Gamma$ denote $\Gamma_1\cup\Gamma_2\cup\Gamma_3$.
Partition~$\pdefs = \{ \pdefc_1, \pdefc_2, \pdefc_3 \}$ is separable on~$\G_{\pdefs,\Psi_3}(\Gamma)$.
We also can see that~$\Gamma_i$ is negative on~$\pdefc_j$ under context~$\Psi_3$ for all~$j \neq i$.
Therefore,
by Theorem~\ref{thm:theory.splitting}, the $\pdefc$\nobreakdash-stable models of~$\Gamma$ are the models of~$\Psi_3$ that are $\pdefc_i$\nobreakdash-stable models of~$\Gamma_i$ for all~$i \in \{1,2,3 \}$.
The intent of context~$\Psi_3$ is to carry information from one part of the theory into another.
In our example all models of~$\Psi_3$ are~$\pdefc_3$\nobreakdash-stable models of~$\Gamma_3$.
Hence, we can simply say that the $\pdefc$\nobreakdash-stable models of~$\Gamma$ are the interpretations that are $\pdefc_i$\nobreakdash-stable models of~$\Gamma_i$ for all~$i \in \{1,2,3 \}$.

It is worth noting that the empty theory approximates any theory.
When such theory is considered in the presented theorem, it more closely resembles the Splitting Theorem by~\citet{felelipa09a}.
In general, we can use the theory itself, its completion~\cite{feleli11a,fanlif22a} or the completion of a part of it as a ``more precise'' approximator.

\section{Conclusions}
The concept of intensionality statements introduced here provides us with a new granularity on considering semantics of  logic programs and its subcomponents.
It also paves a way to the refinement of earlier versions of the Spltting method.
We  generalized the conditions under which this method can be applied to first\nobreakdash-order theories and show how the resulting approach covers more programs commonly used in practice. 
This generalization comes at a price.
The conditions of the Splitting Theorem by~\citet{felelipa09a} are syntactic, while our conditions rely on verification of semantic properties.
In fact, deciding whether there is an edge in our dependency graph is, in general, undecidable.
For instance, a program containing rules \hbox{$p \leftarrow q, t = 0$} and~$q \leftarrow p$, where~$t$ is a polynomial, can be split into subprograms each containing of one of these rules only if the Diophantine equation in the body has no solutions.
However, we illustrated that for many practical problems the Splitting result of this paper is applicable. 
In the future, we will investigate the  possibility to utilize  first\nobreakdash-order theorem provers for 
checking the required semantic conditions.

%% file: appendix.tex
\input{splitting_local.tex}

\input{grounding_splitting.tex}
\input{splitting_correspondence.tex}

\input{splitting_theories_proof.tex}


%% file: splitting_local.tex
\section{Locally Splittable Theories} \label{sec:splitting.local}

This section introduces auxiliary definitions and results that will allow us to prove the Splitting Theorem stated in the main part of the paper.

Given a subset~$\AAA$ of~$\At{I}$,
we say that~$I$ is a $\AAA$\nobreakdash-\emph{stable model} of~$\Gamma$ if it is a stable model of~$\Gamma \cup \EM{I,\AAA}$
where~$\EM{I,\AAA}$ is the set containing disjunction~$p{(\boldd^*) \vee \neg p(\boldd^*)}$ for every ground atom~$p(\boldd^*)$ in~$\At{I} \setminus \AAA$.
We can relate $\AAA$\nobreakdash-stable models with $\pdef$\nobreakdash-stable models as follows.
By~$\At{I,\pdef}$ we denote the set of ground atoms of the form~$p(\boldd^*)$ such that
\begin{gather*}
  I \models p(\boldd^*) \wedge \lambda^p(\boldd^*)
\end{gather*}
where~$p$ is a predicate symbol and $\boldd$ is a tuple of elements of domains of~$I$.

\begin{proposition}\label{prop:grounding.intensional.stable.models}
$I$ is a $\pdef$-stable model of~$\Gamma$ iff $I$ is a $\At{I,\pdef}$-stable model of~$\Gamma$.
\end{proposition}

\begin{proof}
  It is enough to show that~$\tuple{\HH,I} \modelsht \EM{\pdef}$ iff~$\tuple{\HH,I} \modelsht \EM{\At{I,\pdef}}$.
  Furthermore, for this is enough to show that, for every predicate symbol~$p$,
  \begin{gather}
    \tuple{\HH,I} \modelsht \ \ \forall\boldX \left( \neg \pdef^p(\boldX) \to p(\boldX) \vee \neg p(\boldX) \right)
    \label{eq:1:prop:grounding.intensional.stable.models}
  \end{gather}
  iff 
  \begin{gather}
  \text{$\tuple{\HH,I} \modelsht p(\boldd^*) \vee \neg p(\boldd^*)$ for every~$p(\boldd^*)$ in~$\At{I} \setminus \At{I,\pdef}$.}
  \label{eq:1b:prop:grounding.intensional.stable.models}
  \end{gather}
  On the one hand, condition~\eqref{eq:1:prop:grounding.intensional.stable.models} holds iff
  \begin{gather*}
    \tuple{\HH,I} \modelsht \ \ \neg \pdef^p(\boldd^*) \to p(\boldd^*) \vee \neg p(\boldd^*)
  \end{gather*}
  for every tuple of domain elements~$\boldd$.
  This holds iff one of the following conditions hold:
  \begin{align}
    I &\modelsht \pdef^p(\boldd^*)
    \label{eq:2:prop:grounding.intensional.stable.models}
    \\
    \tuple{\HH,I} &\modelsht p(\boldd^*) \vee \neg p(\boldd^*)
    \label{eq:3:prop:grounding.intensional.stable.models}
  \end{align}
  for every tuple of domain elements~$\boldd$.
  Pick some~$p(\boldd^*)$ in~$\At{I}$.
  We proceed by cases. \emph{Case 1.}
  Assume that~$I \not\models p(\boldd^*)$.
  Then, $\tuple{\HH,I} \modelsht \neg p(\boldd^*)$ and this implies~\eqref{eq:3:prop:grounding.intensional.stable.models}.
  Hence, the result holds.
  \emph{Case~2.} Assume that~$I \not\models \lambda^p(\boldd^*)$.
  Then, condition~\eqref{eq:1:prop:grounding.intensional.stable.models} holds iff~\eqref{eq:3:prop:grounding.intensional.stable.models} and the result holds.
  \emph{Case~3.}  Assume that~$I \models p(\boldd^*) \wedge \lambda^p(\boldd^*)$.
  This implies that $p(\boldd^*)$ belongs to~$\At{I,\pdef}$ and, thus, condition~\eqref{eq:1b:prop:grounding.intensional.stable.models} vacuously follow.
  This also implies~\eqref{eq:2:prop:grounding.intensional.stable.models} and, thus, condition~\eqref{eq:1:prop:grounding.intensional.stable.models} follows.
\end{proof}

%



For a
sentence~$F$ over~$\sigma$, an interpretation~$I$ of~$\sigma$ and a subset~$\AAA$ of~$\At{I}$, 
we define the set $\Pos{I}{F}$ of \emph{(strictly) positive atoms of~$F$
with respect to~$I$}.  Elements of this set are atoms of
the signature $\sigma^I$ that have the form $p(\boldd^*)$.\footnote{For
  the definitions of the signature~$\sigma^I$ and names~$d^*$,
  see Section~\ref{sec:smodels}.}
This set is defined recursively, as follows. If $F$
is not satisfied by~$I$
then $\Pos{I}F = \emptyset$ (note how $\bot$ is such a formula). Otherwise,
\begin{itemize}
\item $\Pos{I}{p(\boldt)} = \{p((\boldt^I)^*)\}$;
\item $\Pos{I}{F_1 \wedge F_2} = \Pos{I}{F_1 \vee F_2} = \Pos{I}{F_1} \cup
  \Pos{I}{F_2}$;
    \item $\Pos{I}{F_1 \to F_2} = \begin{cases}
      \Pos{I}{F_2}  &\text{if } I \models F_1
      \\
      \emptyset &\text{otherwise}
    \end{cases}$;
    \item $\Pos{I}{\forall X F(X)} = \Pos{I}{\exists X F(X)} =
      \bigcup_{d\in|I|^s} \Pos{I}{F(d^*)}$ if~$X$ is a variable of sort~$s$.
\end{itemize}
We also define the sets~$\Pnn{I}{F}$ and~$\Nnn{I}{F}$ of \emph{positive nonnegated and negative nonnegated atoms of~$F$ with respect to~$I$}.
These sets are defined recursively, as follows.
If $F$ is not satisfied by~$I$
then $\Pnn{I}{F} = \Nnn{I}{F} = \emptyset$. Otherwise, 
\begin{itemize} 
\item $\Pnn{I}{p(\boldt)} = \{p((\boldt^I)^*)\}$;
\item $\Pnn{I}{F_1 \wedge F_2} = \Pnn{I}{F_1 \vee F_2} = \\ \Pnn{I}{F_1} \cup
\Pnn{I}{F_2}$;
\item $\Pnn{I}{F_1 \to F_2} =\\  \begin{cases}
  \Nnn{I}{F_1} \cup \Pnn{I}{F_2} &\text{if } I \models F_1
  \\
  \emptyset &\text{otherwise}
\end{cases}$;
\item $\Pnn{I}{\forall X F(X)} = \Pnn{I}{\exists X F(X)} = \\
\bigcup_{d\in|I|^s} \Pnn{I}{F(d^*)}$ if~$X$ is a variable of sort~$s$.

\item $\Nnn{I}{p(\boldt)} = \emptyset$;
\item $\Nnn{I}{F_1 \wedge F_2} = \Nnn{I}{F_1 \vee F_2} = \\ \Nnn{I}{F_1} \cup
\Nnn{I}{F_2}$;
\item $\Nnn{I}{F_1 \to F_2} = \\  \begin{cases}
  \Pnn{I}{F_1} \cup \Nnn{I}{F_2} &\text{if } I \models F_1
  \\
  \emptyset &\text{otherwise}
\end{cases}$;
\item $\Nnn{I}{\forall X F(X)} = \Nnn{I}{\exists X F(X)} = \\
\bigcup_{d\in|I|^s} \Nnn{I}{F(d^*)}$ if~$X$ is a variable of sort~$s$.
\end{itemize}
Recall that a strictly positive occurrence of an implication~$F_1 \to F_2$ in a formula~$F$ is called
a \emph{rule} of~$F$.
The \emph{rules} of a theory~$\Gamma$ are all the rules of its sentences.
An \emph{instance} of a rule $F_1 \to F_2$ for
an interpretation~$I$ of~$\sigma$ is a sentence over $\sigma^I$ that can
be obtained from $F_1 \to F_2$ by substituting names~$d^*$ for its free variables.

Given a set~$\Gamma$ of sentences, the \emph{graph of positive dependencies with respect to an interpretation~$I$ and set~$\AAA$ of  atoms},
$\G_{I,\AAA}(\Gamma)$ is the directed graph defined as follows:
\begin{itemize}
  \item Its vertices are elements of~$\AAA$.
  \item It has an edge from~$p(\boldd^*)$ to~$q(\bolde^*)$ iff, for some instance~$F_1 \to F_2$ of a rule of~$\Gamma$,
  $p(\boldd^*) \in \Pos{I}{F_2}$ and
  $q(\bolde^*) \in \Pnn{I}{F_1}$.
\end{itemize}
If~$\Gamma$ is a singleton~$\{F\}$, then we write~$\G_{I,\AAA}(F)$ instead of~$\G_{I,\AAA}(\{F\})$.
We say that a partition~$\{ \AAA_1, \AAA_2 \}$ of a set of atoms~$\AAA \subseteq \At{I}$ is \emph{separable} on~$\G_{I,\AAA}(\Gamma)$ if 
%
every infinite walk~$v_1,v_2,\dotsc$ of~$\G_{I,\AAA}(\Gamma)$ visits almost one~$\AAA_i$ infinitely many times, that is, 
there is~$i \in \{1,2\}$ such that
$\{ k \mid v_k \in \AAA_i \}$ is finite.

A formula~$F$ is said to be \emph{negative} on an interpretation~$I$ and a subset~$\AAA$ of~$\At{I}$ if~$\Pos{I}{F}$ and~$\AAA$ are disjoint.
A set of sentences~$\Gamma$ is said to be \emph{negative} on an interpretation~$I$ and a subset~$\AAA$ if it is negative on all its sentences.

%% file: grounding_splitting.tex
\subsection{Proof of the Splitting Lemma}

\begin{lemma}\label{lem:splitting.if}
Let~$\Gamma$ be a first\nobreakdash-order theory, 
$I$ be an interpretation
and~$\AAA$ be a subset of~$\At{I}$.
Let~$\{ \AAA_1,\AAA_2\}$ be a partition of~$\AAA$.
If $I$ is a $\AAA$\nobreakdash-stable model of~$\Gamma$, 
then $I$ is a $\AAA_i$-stable model of~$\Gamma$ for all~$1 \leq i \leq 2$.
\end{lemma}

\begin{proof}
Assume that~$I$ is a $\AAA$\nobreakdash-stable model of~$\Gamma$ and pick any strict subset~$\HH$ of~$\AAA_i$.
Since~${\AAA_i \subseteq \AAA}$, it follows that~$\HH$ is also a strict subset of~$\AAA$.
Furthermore, since~$I$ is a $\AAA$\nobreakdash-stable model of~$\Gamma$, 
it follows that~${\tuple{\HH,I} \not\modelsht \Gamma \cup \EM{I,\AAA}}$.
Furthermore, ${\AAA_i \subseteq \AAA}$ implies~${\EM{I,\AAA_i} \supseteq \EM{I,\AAA}}$
and, thus, this implies~${\tuple{\HH,I} \not\modelsht \Gamma \cup \EM{I,\AAA_i}}$.
Consequently, $I$ is a $\AAA_i$\nobreakdash-stable model of~$\Gamma$.
\end{proof}

The following result by~\citet{fanlif22a} will be useful in the next Lemma.

\begin{proposition}\label{prop:properties}
\phantom{.}
\begin{itemize}
\item[(a)]
If $\langle\HH,I \rangle\modelsht F$ then
\hbox{$I \models F$}.
\item[(b)]
For any sentence~$F$ that does not contain intensional symbols,
$\langle\HH,I \rangle\modelsht F$ iff ${I\models F}$.
\item[(c)]
For any subset~$\SSS$
of~$\HH$ such that the predicate symbols of its members do not occur in $F$,
$\langle\HH \setminus \SSS, I\rangle\modelsht F$
iff  $\langle\HH, I\rangle\modelsht F$.
\end{itemize}
\end{proposition}

\begin{lemma}\label{l1}
    For any HT\nobreakdash-interpretation $\langle \HH,I \rangle$ and any
    sentence~$F$ over~$\sigma^I$, if $I \models F$ and 
  $\Pos{I}F\subseteq \HH$
    then $\langle \HH,I \rangle\modelsht F.$
\end{lemma}

\begin{proof}
By induction on the size of~$F$.
\emph{Case~1:}~$F$ is $p(\boldt)$.
Then the assumption
$\Pos{I}F\subseteq \HH$ and the
claim $\langle\HH,I\rangle\modelsht F$
turn into the condition $p((\boldt^I)^*)\in \HH$.
\emph{Case~2:}~$F$ is $F_1\land F_2$.  Then from the assumption
${I\models F}$ we conclude that ${I\models F_i}$
for $i=1,2$.  On the other hand,
$$\Pos{I}{F_i}\subseteq\Pos{I}F\subseteq \HH.$$
By the induction hypothesis, it
follows that \hbox{$\langle\HH,I\rangle\modelsht F_i$},
and consequently ${\langle\HH,I\rangle\modelsht F}$.
\emph{Case~3:}~$F$ is \hbox{$F_1\lor F_2$}.  Similar to Case~2.
\emph{Case~4:}~$F$ is $F_1\to F_2$.
Since ${I\models F_1\to F_2}$,
we only need to check that
${\langle\HH,I\rangle \not\modelsht F_1}$
or
${\langle\HH,I\rangle\modelsht F_2}$.
\emph{Case~4.1:} $I\models F_1$.
Since $I\models F_1\to F_2$, it follows that
\hbox{$I\models F_2$}.
On the other hand,
$$\Pos{I}{F_2}=\Pos{I}F\subseteq \HH.$$
By the induction hypothesis, it follows that
$\langle\HH,I\rangle\modelsht F_2$.
\emph{Case~4.2:} $I\not\models F_1$.
By Proposition~\ref{prop:properties}(a), it follows that
$\langle\HH,I\rangle\not\modelsht F_1$.
\emph{Case~5:} $F$ is $\forall X\,G(X)$, where~$X$ is a variable of
sort~$s$.  Then, for every element~$d$ of~$|I|^s$, $I\models G(d^*)$
and
$$\Pos{I}{G(d^*)}\subseteq \Pos{I}F\subseteq \HH.$$
By the induction hypothesis, it follows that
\hbox{$\langle\HH,I\rangle\modelsht G(d^*)$}.  Consequently
$\langle\HH,I\rangle\modelsht \forall X\,G(X)$.
\emph{Case~6:} $F$ is $\exists X\,G(X)$.  Similar to Case~5.
\end{proof}

\begin{lemma}\label{lem:splitting.aux1}
If~$\AAA$ is disjoint from~$\Pos{I}{F}$ and~$I \models F$,
then~${\langle \At{I} \setminus \AAA,I \rangle\modelsht F}$.
\end{lemma}

\begin{proof}
Since~$\AAA$ is disjoint from~$\Pos{I}{F}$
and~$\Pos{I}F \subseteq \At{I}$,
it follows that~$\Pos{I}{F} \subseteq \At{I} \setminus \AAA$.
Since~$I \models F$,
the result follows from Lemma~\ref{l1}.%
\end{proof}






\begin{lemma}\label{lem:splitting.aux2}
Let~$I$ be an interpretation, $\BB_1$ and~$\BB_2$ be subsets of~$\At{I}$ and~$F$ be a sentence.
\begin{enumerate}
\item If~$\BB_2$ is disjoint from~$\Pnn{I}{F}$ and~$\tuple{\HH_1, I} \modelsht F$, then~$\tuple{\HH_2, I} \modelsht F$,

\item If~$\BB_2$ is disjoint from~$\Nnn{I}{F}$ and~$\tuple{\HH_2, I} \modelsht F$, then~$\tuple{\HH_1, I} \modelsht F$,
\end{enumerate}
where~$\HH_1 = \At{I} \setminus \BB_1$ and~$\HH_2 = \At{I} \setminus (\BB_1 \cup \BB_2)$.
\end{lemma}

\begin{proof}
The proof is by structural induction on~$F$.
\\[5pt]
In case that~$F$ is an atom~$p(\boldt)$.
Since~$\tuple{\HH_i, I} \modelsht F$,
it follows that~$p(\boldt^I)$ belongs to~$\HH_i$.
\begin{enumerate}
  \item Then, $p(\boldt^I)$ does not belong to~$\BB_1$.
  Since~$\BB_2$ is disjoint from~$\Pnn{I}{F}$, this means that
  $\BB_1 \cup \BB_2$ is disjoint from~$\Pnn{I}{F} = \Pos{I}{F}$,
  and, thus, condition~1 follows directly from Lemma~\ref{l1}.

  \item Since~$p(\boldt^I)$ belongs to~$\HH_2 \subseteq \HH_1$, it follows that~$\tuple{\HH_1, I} \modelsht p(\boldt)$.
  Therefore, condition~2 is also satisfied.
\end{enumerate}
In case that~$F$ is an implication of form~$F_1 \to F_2$.
\begin{enumerate}
    \item
    Assume that~$\BB_2$ is disjoint from~$\Pnn{I}{F}$ and~$\tuple{\HH_1, I} \modelsht F$.
    The latter implies~$I \models F_1$ and, thus,
    $\Pnn{I}{F} = \Nnn{I}{F_1} \cup \Pnn{I}{F_2}$.
    Then,
    \begin{itemize}
    \item $\BB_2$ is disjoint from~$\Nnn{I}{F_1}$, and
    
    \item $\BB_2$ is disjoint from~$\Pnn{I}{F_2}$.
    \end{itemize}

    Furthermore, $\tuple{\HH_1, I} \modelsht F$
    implies that~${\tuple{\HH_1,I} \not\modelsht F_1}$ or~${\tuple{\HH_1,I} \modelsht F_2}$.
    By induction hypothesis, this implies that~${\tuple{\HH_2,I} \not\modelsht F_1}$ or~${\tuple{\HH_2,I} \modelsht F_2}$.
    In addition, ${\tuple{\HH_1,I} \modelsht F}$ also implies~${I \models F_1 \to F_2}$.
    Therefore, ${\tuple{\HH_2,I} \modelsht F}$.

    \item Assume that~$\BB_2$ is disjoint from~$\Nnn{I}{F}$ and~$\tuple{\HH_1, I} \modelsht F$.
    The latter implies~$I \models F_1$ and, thus,
    $\Pnn{I}{F} = \Pnn{I}{F_1} \cup \Nnn{I}{F_2}$.
    Then,
    \begin{itemize}
    \item $\BB_2$ is disjoint from~$\Pnn{I}{F_1}$, and
    
    \item $\BB_2$ is disjoint from~$\Nnn{I}{F_2}$.
    \end{itemize}
    Furthermore, $\tuple{\HH_2, I} \modelsht F$
    implies~${\tuple{\HH_2,I} \not\modelsht F_1}$ or~${\tuple{\HH_2,I} \modelsht F_2}$.
    By induction hypothesis, this implies that~${\tuple{\HH_1,I} \not\modelsht F_1}$ or~${\tuple{\HH_1,I} \modelsht F_2}$.
    In addition, ${\tuple{\HH_2,I} \modelsht F}$ also implies~${I \models F_1 \to F_2}$.
    Therefore, ${\tuple{\HH_1,I} \modelsht F}$.
\end{enumerate}
Finally, the cases in which~$F$ is of forms~$F_1 \wedge F_2$, $F_1 \vee F_2$, $\forall X\, F_1(X)$ and~$\exists X\, F_1(X)$ follow directly by induction hypothesis.
\end{proof}

  

\begin{lemma}\label{lem:splitting.aux}
Let~$I$ be an interpretation, $\{ \BB_1 ,  \BB_2 \}$ be partition of some set of ground atoms~$\XX$ and~$F$ be a sentence
such that there are no edges from~$\BB_1$ to~$\BB_2$ in~$\G_{I,\XX}(F)$.
Let~$\HH_1 = \At{I} \setminus \BB_1$ and~$\HH_2 = \At{I} \setminus (\BB_1 \cup \BB_2)$.
If~$\tuple{\HH_2,I} \modelsht F$, then~${\tuple{\HH_1, I} \modelsht F}$.
\end{lemma}

\begin{proof}
The proof is by structural induction on~$F$.
%
In case that~$F$ is an atom~$p(\boldt)$, the result follows because~$\HH_2 \subseteq \HH_1$.
\\[5pt]
In case that~$F$ is of form~${F_1 \to F_2}$.
Then, from~${\tuple{\HH_2,I} \modelsht F_1 \to F_2}$ we can conclude that~${I \models F_1 \to F_2}$.
Assume that~$\tuple{\HH_1,I} \modelsht F_1$.
Then, $I \models F_1$ and $I \models F_2$ follows.
We just need to show that~$\tuple{\HH_1,I} \modelsht  F_2$.
We proceed by cases.
\begin{itemize}
  \item If~$\BB_1$ is disjoint from~$\Pos{I}{F_2}$
   then, by Lemma~\ref{lem:splitting.aux1} and the fact~$I \models F_2$, we get~$\tuple{\HH_1,I} \modelsht F_2$.

  \item Otherwise, since~$\BB_1$ is not disjoint from~$\Pos{I}{F_2}$, it must be that~$\BB_2$ is disjoint from~$\Pnn{I}{F_1}$, because there are no edges from~$\BB_1$ to~$\BB_2$ in~$\G_{I,\XX}(F)$.
  Then, by Lemma~\ref{lem:splitting.aux2} and~$\tuple{\HH_1,I} \modelsht F_1$, we get~${\tuple{\HH_2, I} \modelsht F_1}$.
  Since~$\tuple{\HH_2, I} \modelsht F_1 \to F_2$, this implies that $\tuple{\HH_2, I} \modelsht F_2$.
  Then, by induction hypothesis, we get~${\tuple{\HH_1,I} \modelsht F_2}$.
  Note that~$\G_{I,\XX}(F_2)$ is a subgraph of~$\G_{I,\XX}(F)$.  
\end{itemize}
Finally, the cases in which~$F$ is of forms~$F_1 \wedge F_2$, $F_1 \vee F_2$, $\forall X \, F(X)$ and~$\exists X \, F(X)$ follow directly by induction hypothesis.
\end{proof}

\begin{lemma}\label{lem:splitting.v1}
Let~$\Gamma$ be a first\nobreakdash-order theory, 
$I$ be an interpretation
and~$\AAA$ be a subset of~$\At{I}$.
Let~$\{ \AAA_1,\AAA_2\}$ be a partition of~$\AAA$ satisfying the following condition:%
{\setlength{\leftmargini}{5pt}%
\begin{itemize}
  \item[]  for every non\nobreakdash-empty subset~$\XX$ of~$\AAA$, 
  there is some~$1 \leq i \leq 2$ and non\nobreakdash-empty subset~$\BB$ of ~$\XX \cap \AAA_i$ such that there are no edges from~$\BB$ to~$\XX \setminus \BB$ in~$\G_{I,\XX}(\Gamma)$.
\end{itemize}}%
\noindent
If $I$ is a $\AAA_i$-stable model of~$\Gamma$ for all~$1 \leq i \leq 2$,
then
$I$ is a $\AAA$\nobreakdash-stable model of~$\Gamma$.
\end{lemma}

\begin{proof}
Assume that~$I$ is a $\AAA_i$\nobreakdash-stable model of~$\Gamma$ for all~${1 \leq i \leq 2}$ and pick any strict subset~$\HH$ of~$\At{I}$.
It is enough to show that~${\tuple{\HH,I} \not\modelsht \Gamma \cup \EM{I,\AAA}}$.
Let~$\XX = \At{I} \setminus \HH$.
If $\XX$ is not a subset of~$\AAA$, then~${\tuple{\HH,I} \not\modelsht \EM{I,\AAA}}$.
Hence, we assume without loss of generality that~$\XX \subseteq \AAA$ and we show that~${\tuple{\HH,I} \not\modelsht \Gamma}$.

By hypothesis,
there is some~$1 \leq i \leq 2$ and non\nobreakdash-empty subset~$\BB$ of~$\XX \cap \AAA_i$ such that there are no edges from~$\BB$ to~$\XX \setminus \BB$ in~$\G_{I,\XX}(\Gamma)$.
Let~$\HH_i = \At{I} \setminus \BB$ and~$\CC = \XX \setminus \BB$.
Since~$\BB \subseteq \XX$, we get that~$\XX = \BB \cup \CC$.
We also get that
\begin{gather}
\text{there are no edges from~$\BB$ to~$\CC$ in~$\G_{I,\XX}(\Gamma)$.}
\label{eq:1:lem:splitting.v1}
\end{gather}
Furthermore, since~$\BB$ is non-empty, it follows that~${\HH_i \subset \At{I}}$ and, since~$I$ is a $\AAA_i$\nobreakdash-stable model of~$\Gamma$, we immediately get that~${\tuple{\HH_i,I} \not\modelsht \Gamma \cup \EM{I,\AAA_i}}$.
Furthermore, since~$\BB \subseteq \XX \cap \AAA_i \subseteq \AAA_i$,
it follows that~${\tuple{\HH_i,I} \modelsht \EM{I,\AAA_i}}$.
Therefore,
\begin{gather}
\tuple{\HH_i,I} \textstyle\not\modelsht \Gamma.
\label{eq:2:lem:splitting.v1}
\end{gather}
By Lemma~\ref{lem:splitting.aux},
facts~\eqref{eq:1:lem:splitting.v1} and~\eqref{eq:2:lem:splitting.v1}
imply that
\begin{gather}
\tuple{\At{I} \setminus (\BB \cup \CC) , I} \textstyle\not\modelsht \Gamma.
\label{eq:3:lem:splitting.v1}
\end{gather}
Since~${\HH \subseteq \At{I}}$ and~$\XX = \At{I} \setminus \HH$, we get that $\HH = \At{I} \setminus \XX = \At{I} \setminus (\BB \cup \CC)$.
Therefore,
we can rewrite~\eqref{eq:3:lem:splitting.v1} as~$\tuple{\HH,I} \not\modelsht \Gamma$.
\end{proof}

\begin{lemma}\label{lem:separable.partition}
Let~$\{ \XX_1,\XX_2\}$ be a separable partition of some non\nobreakdash-empty~$\XX$.
Then, there is a vertex~$v \in \XX$ such that the set of vertices reachable~$\G_{I,\XX}(\Gamma)$ from~$v$ is a subset of~$\XX_i$ for some~$1 \leq i \leq 2$.
\end{lemma}

\begin{proof}
%
Suppose, for the sake of contradiction,
that for every vertex~$v \in \XX$, the set of vertices reachable from vertex~$v$ in~$\G_{I,\XX}(\Gamma)$ is neither a subset of~$\XX_1$ nor~$\XX_2$.
That is, for every~$v_{i,1} \in \XX_1$, there is~$v_{i,k_i} \in \XX_2$ such that~$v_{i,1},\dotsc,v_{i,k_i}$ is a walk in~$\G_{I,\XX}(\Gamma)$; and for every~$w_{j,1} \in \XX_1$, there is~$w_{j,l_j} \in \XX_2$ such that~$w_{j,1},w_{j,2},\dotsc,w_{j,l_j}$ is a walk in~$\G_{I,\XX}(\Gamma)$.
By taking vertices that satisfy~${v_{i,k_i} = w_{i,1}}$ and~${w_{j,l_j} = v_{i+1,1}}$, we can build the infinite walk
$$v_{1,1},v_{1,2},\dotsc,v_{1,k_1}, w_{1,2},\dotsc,w_{1,l_1}, v_{2,2}, \dotsc$$
that visits both~$\XX_1$ and~$\XX_2$ infinitely many times.
This is a contradiction with the fact that~$\{ \XX_1,\XX_2\}$ is separable.
Hence, there is a vertex~$v \in \XX$ such that the set of vertices reachable~$\G_{I,\XX}(\Gamma)$ from~$v$ is a subset of~$\XX_i$ for some~$1 \leq i \leq 2$.%
\end{proof}

\begin{lemma}\label{lem:separable.partition.X}
Let~$\{ \AAA_1,\AAA_2\}$ be a separable partition of~$\AAA$.
Then,
{\setlength{\leftmargini}{5pt}%
\begin{itemize}
  \item[]  for every non\nobreakdash-empty subset~$\XX$ of~$\AAA$, 
  there is some~$1 \leq i \leq 2$ and non\nobreakdash-empty subset~$\BB$ of ~$\XX \cap \AAA_i$ such that there are no edges from~$\BB$ to~$\XX \setminus \BB$ in~$\G_{I,\XX}(\Gamma)$.
\end{itemize}}%
\end{lemma}

\begin{proof}
Let~$\XX_i = \XX \cap \AAA_i$.
Then, $\{ \XX_1,\XX_2\}$ is a separable partition of~$\XX$
and, by Lemma~\ref{lem:separable.partition},
there is a vertex~$v \in \XX$ such that the set of vertices~$\BB$ reachable~$\G_{I,\XX}(\Gamma)$ from~$v$ is a subset of~$\XX_i$ for some~$i$.
Clearly, there are no edges from~$\BB$ to~$\XX \setminus \BB$ in~$\G_{I,\XX}(\Gamma)$.
\end{proof}

\begin{lemma}[Splitting lemma with infinitary graph]\label{lem:splitting.local}
Let~$\Gamma$ be a first\nobreakdash-order theory, 
$I$ be an interpretation
and~$\AAA$ be a subset of~$\At{I}$.
If~$\{ \AAA_1,\AAA_2 \}$ is a partition of~$\AAA$ that is separable on~$\G_{I,\AAA}(F)$,
then the following two statements are equivalent
\begin{enumerate}
    \item $I$ is a $\AAA$\nobreakdash-stable model of~$\Gamma$, and
    \item $I$ is a $\AAA_i$-stable model of~$\Gamma$ for all~$1 \leq i \leq 2$.
\end{enumerate}
\end{lemma}

\begin{proof}
Condition~1 implies condition~2 by Lemma~\ref{lem:splitting.if}.
Furthermore, by Lemma~\ref{lem:separable.partition.X}, if follows that
{\setlength{\leftmargini}{5pt}%
\begin{itemize}
  \item[]  for every non\nobreakdash-empty subset~$\XX$ of~$\AAA$, 
  there is some~$1 \leq i \leq 2$ and non\nobreakdash-empty subset~$\BB$ of ~$\XX \cap \AAA_i$ such that there are no edges from~$\BB$ to~$\XX \setminus \BB$ in~$\G_{I,\XX}(\Gamma)$.
\end{itemize}}%
\noindent
Then, by Lemma~\ref{lem:splitting.v1}, it follows that Condition~2 implies condition~1.
\end{proof}

\begin{lemma}\label{lem:aux:thm.splitting}
Let~$\Gamma_1,\Gamma_2$ be two first\nobreakdash-order theories.
If~$\AAA$ is a set of ground atoms that is disjoint from~$\Pos{I}{\Gamma_2}$,
then $I$ is an $\AAA$\nobreakdash-stable model of~$\Gamma_1 \cup \Gamma_2$ iff
$I$ is an $\AAA$\nobreakdash-stable model of~$\Gamma_1$ and~$I \models \Gamma_2$.
\end{lemma}

\begin{proof}
For the if direction.
Assume that $I$ is an $\AAA$\nobreakdash-stable model of~$\Gamma_1$ and~$I \models \Gamma_2$.
We get that~$I \models \Gamma_1 \cup \Gamma_2$.
Furthermore since $I$ is an $\AAA$\nobreakdash-stable model of~$\Gamma_1$, for any~$\HH \subset \At{I}$, we get that~$\tuple{\HH,I} \not\models \Gamma_1$ and, thus, $\tuple{\HH,I} \not\models \Gamma_1 \cup \Gamma_2$.
Therefore, $\AAA$\nobreakdash-stable model of~$\Gamma_1 \cup \Gamma_2$.
\\[10pt]
For the only if direction.
Assume that~$I$ is a $\AAA$\nobreakdash-stable model of~$\Gamma_1 \cup \Gamma_2$.
Then~${I \models \Gamma_1}$ and~${I \models \Gamma_2}$.
Furthermore, since~$\AAA$ is disjoint from~$\Pos{I}{\Gamma_2}$,
by Lemma~\ref{l1} it follows that~$\tuple{\HH,I} \not\models \Gamma_1$ for any~$\HH \subset \At{I}$.
Hence, $I$ is an $\AAA$\nobreakdash-stable model of~$\Gamma_1$.
\end{proof}

\begin{theorem}[Splitting Theorem with infinitary graph]\label{thm:splitting.local}
Let~$\Gamma = \Gamma_1 \cup \Gamma_2$ be a set of sentences,
$I$ be an interpretation,
and~$\AAA$ be a subset of~$\At{I}$,
and $\{ \AAA_1, \AAA_2 \}$ be a partition of~$\AAA$ such that
\begin{itemize}
    \item $\{ \AAA_1, \AAA_2 \}$ is separable on~$\G_{I,\AAA}(\Gamma)$;
    \item $\Gamma_1$ is negative on~$\AAA_2$; and
    \item $\Gamma_2$ is negative on~$\AAA_1$.
\end{itemize}
Then the following two statements are equivalent
\begin{enumerate}
    \item $I$ is a $\AAA$\nobreakdash-stable model of~$\Gamma$, and
    \item $I$ is a $\AAA_i$-stable model of~$\Gamma_i$ for all~$i \in \{1,2\}$.
\end{enumerate}
\end{theorem}

\begin{proof}
  By Lemma~\ref{lem:splitting.local} and the fact that~$\Gamma = \Gamma_1 \cup \Gamma_2$, 
  it follows that condition~1 below is equivalent to the conjunction of conditions~a and~b:
  \begin{enumerate}
      \item $I$ is a $\AAA$\nobreakdash-stable model of~$\Gamma$;
      \item[a.] $I$ is a $\AAA_1$-stable model of~$\Gamma_1 \cup \Gamma_2$;
      \item[b.] $I$ is a $\AAA_2$-stable model of~$\Gamma_1 \cup \Gamma_2$.
  \end{enumerate}
  Furthermore, since~$\Gamma_1$ is negative on~$\AAA_2$, we get that~$\Pos{\Gamma_1}{I}$ and~$\AAA_2$ are disjoint.
  Similarly, we get that~$\Pos{\Gamma_2}{I}$ and~$\AAA_1$ are also disjoint.
  By Lemma~\ref{lem:aux:thm.splitting}, this implies condition~a and~b are respectively equivalent to:
  \begin{enumerate}
    \item[a'.] $I$ is a $\AAA_1$-stable model of~$\Gamma_1$ and~$I \models \Gamma_2$, and
    \item[b'.] $I$ is a $\AAA_2$-stable model of~$\Gamma_2$ and~$I \models \Gamma_1$.
  \end{enumerate}
  It remains to see that if~$I$ is a stable model of~$\Gamma$, then it satisfies~$\Gamma$ and, thus, the conjunction of conditions~a' and~b' is equivalent to
  \begin{enumerate}
  \item[2.] $I$ is a $\AAA_i$-stable model of~$\Gamma_i$ for all~$i \in \{1,2\}$.\qed
  \end{enumerate}\let\qed\relax
\end{proof}

%% file: splitting_correspondence.tex
\section{Proof of the Splitting Theorem with Intensionality Statements}

\begin{lemma}\label{lem:partition.correspondence}
If~${\pdefs = \{ \pdef_1, \pdef_2 \}}$ is a partition of~$\pdef$,
then
$\{ \At{I,\pdef_1},\, \At{I,\pdef_2} \}$ is a partition of~$\At{I,\pdef}$.
\end{lemma}

\begin{proof}
Since ${\pdefs = \{ \pdef_1, \pdef_2 \}}$ is a partition of~$\pdef$,
we have
\begin{gather}
    \pdef^p(\boldd^*) \leftrightarrow \pdef_1^p(\boldd^*) \vee \pdef_2^p(\boldd^*)
    \\
    \neg(\pdef_1^p(\boldd^*) \wedge \pdef_2^p(\boldd^*))
    \label{eq:2:lem:partition.correspondence}
\end{gather}
are satisfied by every interpretation~$I$,
every predicate symbol~$p$ and list~$\boldd$ of domain elements of the appropriate length and sorts.
Then,
\begin{align*}
\At{I,\pdef} &= \{\, p(\boldd^*) \mid I \models p(\boldd^*) \wedge \pdef^p(\boldd^*) \,\}
\\
&= \{\, p(\boldd^*) \mid I \models p(\boldd^*) \wedge (\pdef_1^p(\boldd^*) \vee \pdef_2^p(\boldd^*))\,\}
\\
&= \{\, p(\boldd^*) \mid I \models p(\boldd^*) \wedge \pdef_1^p(\boldd^*) \}
\\
&\phantom{=} \ \ \cup 
\{\, p(\boldd^*) \mid I \models p(\boldd^*) \wedge \pdef_2^p(\boldd^*))\,\}
\\
&= \At{I,\pdef_1} \cup \At{I,\pdef_2} 
\end{align*}
Finally, $\At{I,\pdef_1} \cap \At{I,\pdef_2} = \emptyset$
follows from~\eqref{eq:2:lem:partition.correspondence}.
\end{proof}

\begin{lemma}\label{lem:graph.correspondence.vertices}
Let~$\Gamma$ be a set of basic rules.
If~${\pdefs = \{ \pdef_1, \pdef_2 \}}$ is a partition of~$\pdef$
and~$p(\boldd^*)$ is a vertex in~$\G_{I,\At{I,\pdef}}(\Gamma)$ such that~$p(\boldd^*)$ belongs to~$\At{I,\pdef_i}$ with~$i \in \{1,2\}$,
then $(p,\pdef_i)$ is a vertex in~$\G_{\pdefs}(\Gamma)$ and~$I \models \pdef_i^p(\boldd^*) \wedge \neg\pdef_j^p(\boldd^*)$ with~$j \neq i$.
\end{lemma}

\begin{proof}
Since~$\pdefs$ is a partition of~$\pdef$,
by Lemma~\ref{lem:partition.correspondence}, 
$\At{I,\pdef_1}$ and~$\At{I,\pdef_2}$ form a partition of~$\At{I,\pdef}$.
Furthermore,
since~$p(\boldd^*)$ is a vertex in~$\G_{I,\At{I,\pdef}}(\Gamma)$, by definition, we get that~$p(\boldd^*)$ belongs to~$\At{I,\pdef}$.
Hence, $p(\boldd^*)$ belongs to~$\At{I,\pdef_i}$ for exactly one~$i \in \{1,2\}$ and, thus, $I \models p(\boldd^*) \wedge \pdef_i(\boldd^*)$.
We also get that~$I \not\models p(\boldd^*) \wedge \pdef_j(\boldd^*)$ with~$j = \{1,2\}$ such that~$j \neq i$.
Therefore, ${I \models \pdef_i(\boldd^*) \wedge \neg\pdef_j(\boldd^*)}$.
Finally, since~$I \models p(\boldd^*) \wedge \pdef_i(\boldd^*)$, we get~${\exists \boldX \pdef^p_i \not\equiv \bot}$.
This implies that~$(p,\pdef_i)$ is a vertex of~$\G_{\pdefs}(\Gamma)$.
\end{proof}

\noindent
Given a vertex~$p(\boldd^*)$ in~$\G_{I,\At{I,\pdef}}(\Gamma)$, by~$[p(\boldd^*)]$ we denote the vertex~$(p,\pdef_i)$ in~$\G_{\pdefs}(\Gamma)$ such that~$I \models \pdef_i^p(\boldd^*)$.
Note that, by Lemma~\ref{lem:graph.correspondence.vertices}, such vertex exists and it is unique.

\begin{lemma}\label{lem:graph.correspondence.edges}
Let~$\Gamma$ be a set of rules.
If~${\pdefs = \{ \pdef_1, \pdef_2 \}}$ is a partition of~$\pdef$
and~$(v,w)$ is an edge in~$\G_{I,\At{I,\pdef}}(\Gamma)$,
then~$([v],[w])$ is an edge in~$\G_{\pdefs}(\Gamma)$.
\end{lemma}

\begin{proof}
Since~$\Gamma$ is a set of rules and~$(v,w)$ is an edge in~$\G_{I,\At{I,\pdef}}(\Gamma)$, there is an instance~$F_1 \to F_2$ of a rule such that
$v \in \Pos{I}{F_2}$
and
$w \in \Pos{I}{F_1}$.
Let~$v = p(\boldd^*)$ and~$w = q(\bolde^*)$.
Then,
$v \in \Pos{I}{F_2}$
and
$w \in \Pos{I}{F_1}$
imply that
\begin{gather}
    I \models  F_1 \wedge F_2 \wedge p(\boldd^*) \wedge q(\bolde^*)
    \label{eq:1:lem:graph.correspondence.edges}
\end{gather}
Let~$[w] = (p,\pdef_v)$ and~$[v] = (q,\pdef_w)$.
From Lemma~\ref{lem:graph.correspondence.vertices},
we get
$I \models \pdef_v^p(\boldd^*) \wedge \pdef_w^q(\bolde^*)$.
This plus~\eqref{eq:1:lem:graph.correspondence.edges} imply
\begin{gather}
    I \models  F_1 \wedge p(\boldd^*) \wedge \pdef_v^p(\boldd^*) \wedge \pdef_w^q(\bolde^*)
    \label{eq:2:lem:graph.correspondence.edges}
\end{gather}
Let~$B \to H$ be the rule of which~$F_1 \to F_2$ is an instance, $\boldX$ be the free variables in~$B \to H$ and~$\boldx$ be names of domain elements such that~$(B \to H)^\boldX_\boldx = F_1 \to F_2$.
Then, there is 
a positive occurrence of~$p$ in~$H$ within a literal~$p(\boldt)$ 
and 
a positive occurrence of~$q$ in~$B$ within a literal~$q(\boldr)$ 
such that~$\boldd = (\boldt^\boldX_\boldx)^I$ and~$\bolde = (\boldr^\boldX_\boldx)^I$.
Therefore~\eqref{eq:2:lem:graph.correspondence.edges} implies
\begin{gather*}
    I \models  B^\boldX_\boldx \wedge p(\boldt^\boldX_\boldx) \wedge \pdef_v^p(\boldt^\boldX_\boldx) \wedge \pdef_w^q(\boldr^\boldX_\boldx)
\end{gather*}
and, this implies that
\begin{gather*}
    \exists \boldX \big( B \wedge p(\boldt) \wedge \pdef_v^p(\boldt) \wedge \pdef_w^q(\boldr) \big)
\end{gather*}
is satisfiable.
Consequently, $([v],[w])$ is an edge in~$\G_{\pdefs}(\Gamma)$.%
\end{proof}

%% file: splitting_theories_proof.tex
\begin{lemma}\label{lem:pos.grounding}
    Let~$F$ be a sentence.
    Let~$\Psi$ be a theory and~$I$ be an model of~$\Psi$.
    Then, the following conditions hold:
    \begin{enumerate}
        \item If~$p(\boldd^*)$ belongs to~$\Pos{I}{F}$,
        then there is some a strictly positive occurrence~$p(\boldt)$ of predicate symbol~$p$ in~$F$ such that~$I$ satisfies~$\Pos{\Psi}{F}^\boldY_{\boldd^*}$ where~$\boldY$ are the free variables in~$\Pos{\Psi}{F}$.

        \item If~$p(\boldd^*)$ belongs to~$\Pnn{I}{F}$,
        then there is some a positive nonnegated occurrence~$p(\boldt)$ of predicate symbol~$p$ in~$F$ such that~$I$ satisfies~$\Pnn{\Psi}{F}^\boldY_{\boldd^*}$ where~$\boldY$ are the free variables in~$\Pnn{\Psi}{F}$.

        \item If~$p(\boldd^*)$ belongs to~$\Nnn{I}{F}$,
        then there is some a negative nonnegated occurrence~$p(\boldt)$ of predicate symbol~$p$ in~$F$ such that~$I$ satisfies~$\Nnn{\Psi}{F}^\boldY_{\boldd^*}$ where~$\boldY$ are the free variables in~$\Nnn{\Psi}{F}$.
    \end{enumerate}
\end{lemma}

\begin{proof}
    The proof follows by induction on the size of the formula.
    Note that the fact that~$p(\boldd^*)$ belongs to~$\Pos{I}{F}$, $\Pnn{I}{F}$ or~$\Nnn{I}{F}$ implies that~$I \models F$.
    Since~$I$ is a model of~$\Psi$, this implies that~$\Psi \cup \{ \exists \boldX\,F \}$ is satisfiable and, thus, we are not in the case where~$\Pos{\Psi}{F}$ (resp. $\Pnn{I}{F}$ or~$\Nnn{I}{F}$) is~$\bot$ by definition.
    \\[5pt]
    \emph{Case~1.} $F$ is an atom, then~$F$ is of the form~$p(\boldt)$ with~${I \models p(\boldt)}$ and~${\boldd = \boldt^I}$.
    Hence, there is a strictly positive occurrence~$p(\boldt)$ of predicate~$p$ in~$F$.
    Any strictly positive occurrence is also a positive nonnegated occurrence.
    Furthermore~$\Pos{\Psi}{F}^\boldY_{\boldd^*} = \Pnn{\Psi}{F}^\boldY_{\boldd^*}$ is~$p(\boldt) \wedge \boldd^* = \boldt$, which is satisfied by~$I$.
    Note that if~$F$ is an atom~$\Nnn{I}{F}$ is the empty set by definition. Therefore, the third condition holds vacuous.
    \\[5pt]
    \emph{Case~2.} $F$ is of the form~$F_1 \wedge F_2$.
    Then, $p(\boldd^*)$ belongs to~$\Pos{I}{F_i}$ for some~${1 \leq i \leq 2}$.
    By induction hypothesis, there is some a positive occurrence~$p(\boldt)$ of predicate symbol~$p$ in~$F_i$ such that~$I$ satisfies~$\Pos{\Psi}{F_i}^\boldY_{\boldd^*}$.
    Remains observing that~$\Pos{\Psi}{F} \equiv \Pos{\Psi}{F_i} \wedge F_j$ with~$1 \leq j \leq 2$ and~$i \neq j$ and that~$I \models F_j$ because~$I \models F$.
    The same reasoning applies to~$\Pnn{I}{F}$ and~$\Nnn{I}{F}$.
    \\[5pt]
    \emph{Case~3.} $F$ is of the form~$F_1 \vee F_2$.
    Then, $p(\boldd^*)$ belongs to~$\Pos{I}{F_i}$ for some~${1 \leq i \leq 2}$.
    By induction hypothesis, this implies that there is some a positive occurrence~$p(\boldt)$ of predicate symbol~$p$ in~$F_i$ such that~$I$ satisfies~$\Pos{\Psi}{F_i}^\boldY_{\boldd^*}$.
    Remains observing that~$\Pos{\Psi}{F} = \Pos{\Psi}{F_i}$.
    The same reasoning applies to~$\Pnn{I}{F}$ and~$\Nnn{I}{F}$.
    \\[5pt]
    \emph{Case~4.} $F$ is of the form~${F_1 \to F_2}$.
    Note that~$p(\boldd^*)$ belonging to~$\Pos{I}{F}$  (resp. $\Pnn{I}{F}$ or~$\Nnn{I}{F}$)  implies that (i)~$I \models F_1$ and that (ii)
    \begin{enumerate}
        \item $p(\boldd^*)$ belongs to~$\Pos{I}{F_2}$,
        \item $p(\boldd^*)$ belongs to~$\Nnn{I}{F_1} \cup \Pnn{I}{F_2}$, and
        \item $p(\boldd^*)$ belongs to~$\Pnn{I}{F_1} \cup \Nnn{I}{F_2}$.
    \end{enumerate}
    Since~$I \models F$, (i) implies that $I \models F_1 \wedge F_2$ and, since~$I$ is a model of~$\Psi$, this implies $F_1^\Psi = F_1$.
    \begin{enumerate}
        \item By induction hypothesis, (ii) implies that there is some a positive occurrence~$p(\boldt)$ of predicate symbol~$p$ in~$F_2$ such that~$I$ satisfies~$\Pos{\Psi}{F_2}^\boldY_{\boldd^*}$.
        Hence, $I$ satisfies 
        $$\Pos{\Psi}{F}^\boldY_{\boldd^*} \ = \ F_1^\Psi \wedge \Pos{\Psi}{F_2}^\boldY_{\boldd^*} \ = \ F_1 \wedge \Pos{\Psi}{F_2}^\boldY_{\boldd^*}$$

        \item By induction hypothesis, (ii) implies that there is some negative nonnegated occurrence of~$p(\boldt)$ of predicate symbol~$p$ in~$F_1$ such that~$I$ satisfies~$\Nnn{\Psi}{F_1}^\boldY_{\boldd^*}$ or some positive nonnegated occurrence of~$p(\boldt)$ of predicate symbol~$p$ in~$F_2$ such that~$I$ satisfies~$\Pnn{\Psi}{F_2}^\boldY_{\boldd^*}$.
        In both cases there is a positive nonnegated occurrence of~$p(\boldt)$ of predicate symbol~$p$ in~$F$.
        Furthermore, if such occurrence happens in~$F_1$, 
        then
        $${\Pnn{\Psi}{F}^\boldY_{\boldd^*} = \Nnn{\Psi}{F_1}^\boldY_{\boldd^*}}$$
        and, thus, condition~2 holds.
        Otherwise, 
        $$\Pnn{\Psi}{F}^\boldY_{\boldd^*} = F_1^\Psi \wedge \Pnn{\Psi}{F_2}^\boldY_{\boldd^*} = F_1 \wedge \Pnn{\Psi}{F_2}^\boldY_{\boldd^*}$$
        and, thus, condition~2 also holds.

        \item This case is analogous to the previous case.
    \end{enumerate}
    \emph{Case~5.} $F$ is of the form~$\forall X G(X)$.
    Then, $p(\boldd^*)$ belongs to~$\Pos{I}{G(d^*)}$ for some domain element of the appropriate sort.
    By induction hypothesis, there is some positive occurrence~$p(\boldr)$ of predicate symbol~$p$ in~$G^X_{d^*}$ such that~$I$ satisfies~$\Pos{\Psi}{G^X_{d^*}}^\boldY_{\boldd^*}$ for every domain element~$d$ of the appropriate sort.
    Then, there is an occurrence~$p(\boldt)$ the occurrence on~$G$ corresponding to~$p(\boldr)$ in~$G^X_{d^*}$ such that~$\Pos{\Psi}{G}^{X\boldY}_{d^*\boldd^*}$ for every domain element~$d$ of the appropriate sort and, thus, $I$ satisfies~$\forall X \Pos{\Psi}{G}^\boldY_{\boldd^*}$.
    \emph{Case~6.} $F$ is of the form~$\exists X G(X)$.
    Then, $p(\boldd^*)$ belongs to~$\Pos{I}{G(d^*)}$ for some domain element of the appropriate sort.
    By induction hypothesis, there is some positive occurrence~$p(\boldr)$ of predicate symbol~$p$ in~$G^X_{d^*}$ such that~$I$ satisfies~$\Pos{\Psi}{G^X_{d^*}}^\boldY_{\boldd^*}$ for some domain element~$d$ of the appropriate sort.
    Then, there is an occurrence~$p(\boldt)$ the occurrence on~$G$ corresponding to~$p(\boldr)$ in~$G^X_{d^*}$ such that~$\Pos{\Psi}{G}^{X\boldY}_{d^*\boldd^*}$ for some domain element~$d$ of the appropriate sort and, thus, $I$ satisfies~$\exists X \Pos{\Psi}{G}^\boldY_{\boldd^*}$.
    The same reasoning applies to~$\Pnn{I}{F}$ and~$\Nnn{I}{F}$.
\end{proof}

\begin{lemma}\label{lem:theory.graph.correspondence.vertices}
    Let~$\Gamma$ and~$\Psi$ be two set theories and~$I$ be a model of~$\Psi$.
    If~${\pdefs = \{ \pdef_1, \pdef_2 \}}$ is a partition of~$\pdef$
    and~$p(\boldd^*)$ is a vertex in~$\G_{I,\At{I,\pdef}}(\Gamma)$ such that~$p(\boldd^*)$ belongs to~$\At{I,\pdef_i}$ with~$i \in \{1,2\}$,
    then $(p,\pdef_i)$ is a vertex in~$\G_{\pdefs,\Psi}(\Gamma)$ and
    $$I \models \pdef_i^p(\boldd^*) \wedge \neg\pdef_j^p(\boldd^*)$$
    with~$j \neq i$.
    \end{lemma}
    
    \begin{proof}
    Since~$\pdefs$ is a partition of~$\pdef$,
    by Lemma~\ref{lem:partition.correspondence}, 
    $\At{I,\pdef_1}$ and~$\At{I,\pdef_2}$ form a partition of~$\At{I,\pdef}$.
    Furthermore,
    since~$p(\boldd^*)$ is a vertex in~$\G_{I,\At{I,\pdef}}(\Gamma)$, by definition, we get that~$p(\boldd^*)$ belongs to~$\At{I,\pdef}$.
    Hence, $p(\boldd^*)$ belongs to~$\At{I,\pdef_i}$ for exactly one~$i \in \{1,2\}$ and, thus, $I \models p(\boldd^*) \wedge \pdef_i(\boldd^*)$.
    We also get that~$I \not\models p(\boldd^*) \wedge \pdef_j(\boldd^*)$ with~$j = \{1,2\}$ such that~$j \neq i$.
    Therefore, ${I \models \pdef_i(\boldd^*) \wedge \neg\pdef_j(\boldd^*)}$ .
    Finally, since~$I \models p(\boldd^*) \wedge \pdef_i(\boldd^*)$ and~$I$ is a model of~$\Psi$, we get~${ \Psi \cup \{\exists \boldX \pdef^p_i \} }$ is satisfiable.
    Therefore~$(p,\pdef_i)$ is a vertex of~$\G_{\pdefs,\Psi}(\Gamma)$.%
\end{proof}

\noindent
Given a vertex~$p(\boldd^*)$ in~$\G_{I,\At{I,\pdef}}(\Gamma)$, by~$[p(\boldd^*)]$ we denote the vertex~$(p,\pdef_i)$ in~$\G_{\pdefs}(\Gamma)$ such that~$I \models \pdef_i^p(\boldd^*)$.
Note that, by Lemma~\ref{lem:theory.graph.correspondence.vertices}, such vertex exists and it is unique.

\begin{lemma}\label{lem:theory.graph.correspondence.edges}
Let~$\Gamma$ and~$\Psi$ be two theories and~$I$ be a model of~$\Psi$.
If~${\pdefs = \{ \pdef_1, \pdef_2 \}}$ is a partition of~$\pdef$
and~$(v,w)$ is an edge in~$\G_{I,\At{I,\pdef}}(\Gamma)$,
then~$([v],[w])$ is an edge in~$\G_{\pdefs,\Psi}(\Gamma)$.
\end{lemma}

\begin{proof}
Since~$(v,w)$ is an edge in~$\G_{I,\At{I,\pdef}}(\Gamma)$, there is a rule~$B \to H$ of~$\Gamma$ with free variables~$\boldX$ and a tuple of domain elements~$\boldx$ such that
$v \in \Pos{I}{H^\boldX_\boldx}$
and
$w \in \Pnn{I}{B^\boldX_\boldx}$.
Let~$v = p(\boldd^*)$ and~$w = q(\bolde^*)$.
By Lemma~\ref{lem:pos.grounding}, this implies that
\begin{itemize}
    \item there is some a strictly positive occurrence~$p(\boldt)$ of predicate symbol~$p$ in~$H$ such that~$I$ satisfies~$\Pos{\Psi}{H^\boldX_\boldx}^\boldY_{\boldd^*}$ where~$\boldY$ are the free variables in~$\Pos{\Psi}{H^\boldX_\boldx}$.

    \item there is some a positive nonnegated occurrence~$q(\bolds)$ of predicate symbol~$q$ in~$B$ such that~$I$ satisfies~$\Pnn{\Psi}{B^\boldX_\boldx}^\boldZ_{\bolde^*}$ where~$\boldZ$ are the free variables in~$\Pnn{\Psi}{B^\boldX_\boldx}$.
\end{itemize}
Let~$[w] = (p,\pdef_v)$ and~$[v] = (q,\pdef_w)$.
From Lemma~\ref{lem:theory.graph.correspondence.vertices},
we get
$I \models \pdef_v^p(\boldd^*) \wedge \pdef_w^q(\bolde^*)$
and, thus, 
\begin{gather*}
    I \models  \Pnn{\Psi}{B^\boldX_\boldx}^\boldY_{\boldd^*} \wedge \Pos{\Psi}{H^\boldX_\boldx}^\boldZ_{\bolde^*} \wedge \pdef_v^p(\boldd^*) \wedge \pdef_w^q(\bolde^*)
\end{gather*}
which implies that
\begin{gather*}
    I \models \exists\boldX\boldY\boldZ \big(
        \Pnn{\Psi}{B} \wedge \Pos{\Psi}{H} \wedge \pdef_v^p(\boldY) \wedge \pdef_w^q(\boldZ)
        \big)
\end{gather*}
Since~$I$ is a model of~$\Psi$, this implies that~\eqref{eq:edge.sentence.arbitrary.formulas} is satisfiable.
Consequently, $([v],[w])$ is an edge in~$\G_{\pdefs}(\Gamma)$.
\end{proof}

\begin{lemma}\label{lem:theory.correspondence.separable}
Let~$\Gamma$ and~$\Psi$ be two theories and~$I$ be a model of~$\Psi$.
If~${\pdefs = \{ \pdef_1, \pdef_2 \}}$ is a partition of~$\pdef$ that is separable on~$\G_{\pdefs,\Psi}(\Gamma)$,
then
$\{ \At{I,\pdef_1},\, \At{I,\pdef_2} \}$ is a partition of~$\At{I,\pdef}$ and it is separable on~$\G_{I,\At{I,\pdef}}(\Gamma)$
for any interpretation~$I$.
\end{lemma}

\begin{proof}
Assume that~$\pdefs$ is a partition of~$\pdef$ that is separable on~$\G_{\pdefs,\Psi}(\Gamma)$.
Then, by Lemma~\ref{lem:partition.correspondence}, we get that~$\{ \At{I,\pdef_1},\, \At{I,\pdef_2} \}$ is a partition of~$\At{I,\pdef}$.
Let us show that this partition is separable on~$\G_{I,\At{I,\pdef}}(\Gamma)$.
Pick any infinite walk~$v_1,v_2,v_3,\dots$ of~$\G_{I,\At{I,\pdef}}(\Gamma)$.
By Lemma~\ref{lem:theory.graph.correspondence.edges}, we get that~$[v_1],[v_2],[v_3],\dots$ is an infinite walk of~$\G_{\pdefs,\Psi}(\Gamma)$.
In addition, by Lemma~\ref{lem:theory.graph.correspondence.vertices}, each~$v_k \in \At{I,\pdef_i}$ satisfies~$[v_k] = (p,\pdef_i)$ for some predicate symbol~$p$.
Hence,
$$
\{ k \mid v_k \in \At{I,\pdef_i} \} \ = \ \{ k \mid [v_k] = (p,\pdef_i) \} 
$$
for~$1 \leq i \leq 2$.
Since~$\pdefs$ is separable on~$\G_{\pdefs,\Psi}(\Gamma)$, at least one of this sets if finite and, therefore,
$\{ \At{I,\pdef_1},\, \At{I,\pdef_2} \}$ is separable on~$\G_{I,\At{I,\pdef}}(\Gamma)$.%
\end{proof}

\begin{lemma}\label{lem:theory.correspondence.negative}
Let~$\Gamma$ and~$\Psi$ be two theories and~$I$ be a model of~$\Psi$.
If~$\Gamma$ is $\Psi$\nobreakdash-negative on~$\pdef$,
then~$\Gamma$ is negative on~$I$ and~$\At{I,\pdef}$.
\end{lemma}

\begin{proof}
Assume that~$\Gamma$ is $\Psi$\nobreakdash-negative in~$\pdef$ and suppose, for the sake of contradiction, that~$\Gamma$ is not negative on~$I$ and~$\At{I,\pdef}$, that is, that there is some ground atom~$p(\boldd^*)$ that belongs to~$\Pos{I}{\Gamma} \cap \At{I,\pdef}$.
Since~$p(\boldd^*)$ that belongs to~$\At{I,\pdef}$, it follows that
\begin{gather*}
    I \models p(\boldd^*) \wedge \lambda^p(\boldd^*)
\end{gather*}
Furthermore, since~$p(\boldd^*)$ belongs to~$\Pos{I}{\Gamma}$, there is a rule~${B \to H}$ of~$\Gamma$ with free variables~$\boldX$ and a tuple~$\boldx$ of names of domain elements of the appropriate length and sorts such that~$I \models B^\boldX_\boldx$ and $p(\boldd^*)$ belongs to~$\Pos{I}{H^\boldX_\boldx}$.
By Lemma~\ref{lem:pos.grounding}, there is a strictly positive occurrence~$p(\boldt)$ of~$p$ in~$H$ such that~$I$ satisfies~$\Pos{\Psi}{H^\boldX_\boldx}^\boldY_{\boldd^*}$ where~$\boldY$ are the free variables in~$\Pos{\Psi}{H^\boldX_\boldx}$.
Hence,
\begin{gather*}
    I \models B^\boldX_\boldx \wedge \Pos{\Psi}{H^\boldX_\boldx}^\boldY_{\boldd^*} \wedge \lambda^p(\boldd^*)
\end{gather*}
which implies that
\begin{gather*}
    I \models \Psi \cup \{ \, \exists \boldX\boldY \bigl(
    B \wedge \Pos{\Psi}{H} \wedge \lambda^p(\boldY)
    \bigr)
    \}
\end{gather*}
This implies that~$\Gamma$ is $\Psi$\nobreakdash-negative, which is a contradiction with the assumption.
Consequently, $\Gamma$ is negative on~$I$ and~$\At{I,\pdef}$.
\end{proof}

\begin{lemma}\label{lem:theory.splitting}
Let~$\Gamma \!=\! \Gamma_1 \!\cup\! \text{\footnotesize\dots} \!\cup \Gamma_n$\! and~$\Psi$ be two theories, and let
and $\pdefs \!=\! \{ \pdef_1, \text{\footnotesize\dots} ,\pdef_n \}$ be a partition of~$\pdef$ such that
\begin{itemize}
    \item $\pdefs$ is separable on~$\G_{\pdefs,\Psi}(\Gamma)$; and
    \item each~$\Gamma_i$ is $\Psi$\nobreakdash-negative on~$\pdef_j$ for all~$j \neq i$.
\end{itemize}
Then, for any interpretation~$I$ that satisfies~$\Psi$,
the following two statements are equivalent
\begin{itemize}
    \item $I$ is a $\pdef$\nobreakdash-stable model of~$\Pi$, and
    \item $I$ is a $\pdef_i$-stable model of~$\Pi_i$ for all~$1\leq i\leq n$.
\end{itemize}
\end{lemma}

\begin{proof}
Assume that~$n = 2$.
The proof for~$n>3$ follows by induction on~$n$.
Pick any model~$I$ of~$\Psi$.
Since~$\pdefs$ is a partition of~$\pdef$ that is separable on~$\G_{\pdefs,\Psi}(\Gamma)$, by Lemma~\ref{lem:theory.correspondence.separable}, it follows that $\{ \At{I,\pdef_1},\, \At{I,\pdef_2} \}$ is a partition of~$\At{I,\pdef}$ and it is separable on~$\G_{I,\At{I,\pdef}}(\Gamma)$
for any interpretation~$I$.
Similarly, since~$\Gamma_1$ is negative on~$\pdef_2$ and $\Gamma_2$ is negative on~$\pdef_1$, by Lemma~\ref{lem:theory.correspondence.negative}, it follows that~$\Gamma_1$ is negative  on~$I$ and~$\At{I,\pdef_2}$ and~$\Gamma_2$ is negative  on~$I$ and~$\At{I,\pdef_1}$.
Therefore, the conditions of Theorem~\ref{thm:splitting.local} are met and, thus, the following two statements are equivalent
\begin{enumerate}
    \item $I$ is a $\AAA$\nobreakdash-stable model of~$\Gamma$, and
    \item $I$ is a $\AAA_i$-stable model of~$\Gamma_i$ for all~$i \in \{1,2\}$.
\end{enumerate}
Finally, the result follows by Proposition~\ref{prop:grounding.intensional.stable.models}.
\end{proof}

\begin{proofof}{Theorem~\ref{thm:theory.splitting}}
    From Proposition~\ref{prop:splitting}, if~$I$ is a $\pdef$\nobreakdash-stable model of~$\Gamma$, then~$I$ is a $\pdef_i$-stable model of~$\Gamma_i$ for all~$1\leq i\leq n$.
    Since~$\Psi$ is an $\pdef$\nobreakdash-approximator of~$\Gamma$, by definition, it also follows that~$I$ is a model of~$\Psi$.
    \emph{The other direction.}
    Pick any model~$I$ of~$\Psi$ such that $I$ a~$\pdef_i$-stable model of~$\Gamma_i$ for all~$1\leq i\leq n$.
    From Lemma~\ref{lem:theory.splitting}, this implies that~$I$ is a $\pdef$\nobreakdash-stable model of~$\Gamma$
\end{proofof}

\section{Splitting Theorem for Disjuntive Logic Programs}

\begin{lemma}\label{lem:SPos.entails.F}
    Let~$F$ be a sentence and~$p(\boldt)$ be a strictly positive occurrence in~$F$.
    Then, $\Pos{\emptyset}{F} \models F \wedge p(\boldt)$.
\end{lemma}

\begin{proof}
    By induction in the size of the formula.
    \emph{Case~1.} If~$F$ is an atom~$p(\boldt)$, then~$\Pos{\emptyset}{F} = p(\boldt) \wedge \boldY = \boldt$ and the result holds.
    \emph{Case~2.}
    If~${F = F_1 \wedge F_2}$, then~$\Pos{\emptyset}{F} = \Pos{\emptyset}{F_1} \wedge \Pos{\emptyset}{F_2} = \Pos{\emptyset}{F_i} \wedge F_j$ with~$F_i$ be subformula containing the occurrence~$p(\boldt)$ and~$j = 3 - i$.
    By induction hypothesis, we get that
    $\Pos{\emptyset}{F_i} \models F_i \wedge p(\boldt)$.
    Therefore, the result holds.
    \emph{Case~3.}
    If~${F = F_1 \vee F_2}$, then~$\Pos{\emptyset}{F} = \Pos{\emptyset}{F_i}$ with~$F_i$ containing the occurrence.
    By induction hypothesis, we get that
    $\Pos{\emptyset}{F_i} \models F_i \wedge p(\boldt)$.
    Therefore, the result holds.
    \emph{Cases~4 and~5.}
    If~$F = \forall X G$ or~$F = \exists X G$, then~$\Pos{\emptyset}{F} = \exists X \Pos{\emptyset}{G}$.
    By induction hypothesis, we get that
    $\Pos{\emptyset}{G^X_x} \models G^X_x \wedge p(\boldt)$.
    Therefore, the result holds.
    \emph{Case~6.}
    If~$F = F_1 \to F_2$, then
    \begin{gather}
    \Pos{\emptyset}{F} = F_1^\emptyset \wedge \Pos{\emptyset}{F_2}
        \label{eq:1:lem:SPos.entails.F}
    \end{gather}
    If~$F_1$ is unsatisfiable, then~\eqref{eq:1:lem:SPos.entails.F}~$F_1^\emptyset$ is~$\bot$ and, thus, $\Pos{\emptyset}{F}$ is equivalent to~$\bot$.
    Therefore, the result holds.
    Otherwise,~$F_1^\emptyset$ is~$F_1$ and, thus, \eqref{eq:1:lem:SPos.entails.F} is~$F_1 \wedge \Pos{\emptyset}{F_2}$.
    By induction hypothesis, we get that
    $\Pos{\emptyset}{F_2} \models F_2 \wedge p(\boldt)$.
    As a result, we can see that
    $\Pos{\emptyset}{F} \models F_1 \wedge F_2 \wedge p(\boldt) \models F \wedge p(\boldt)$
    and the result holds.
\end{proof}

\begin{lemma}\label{lem:Pnn.entails.F}
    Let~$F$ be a implication\nobreakdash-free sentence and~$p(\boldt)$ be a positive nonnegated occurrence in~$F$.
    Then, $\Pnn{\emptyset}{F} \models F \wedge p(\boldt)$.
\end{lemma}

\begin{proof}
    If~$F$ is implication\nobreakdash-free, then~$\Pnn{\emptyset}{F}= \Pnn{\emptyset}{F}$ and the result follows by Lemma~\ref{lem:SPos.entails.F}.
\end{proof}

\begin{lemma}\label{lem:positive.dependency.graph.for.programs}
    Let~$\Pi$ be a logic program and~$\pdefs$ be a partition of some intensionality statement.
    Then, $\G_{\pdefs,\emptyset}(\Pi)$ is a subgraph of $\G_{\pdefs}(\Pi)$.
\end{lemma}

\begin{proof}
    Pick any vertex~$(p,\pdef)$ be a of~$\G_{\pdefs,\emptyset}(\Pi)$.
    Then, $\exists \boldX \pdef^p(\boldX)$ is satisfiable and, thus, $(p,\pdef)$ is a vertex of~$\G_{\pdefs}(\Pi)$.
    Pick an edge from~$(p,\pdef_i)$ to~$(q,\pdef_j)$ in~$\G_{\pdefs,\emptyset}(\Pi)$.
    Then, there is some rule~$B \to H$ in~$\Pi$, a strictly positive occurrence of~$p$ in~$H$ of the form~$p(\boldt)$,
    a positive nonnegated occurrence of~$q$ in~$B$ of the form~$q(\boldr)$ and
    \begin{gather}
      \Psi \cup \big\{ \,  \exists \boldX\boldY\boldZ\, \big( \Pnn{\Psi}{B} \wedge \Pos{\Psi}{H} \wedge \pdef_j^q(\boldY) \wedge \pdef_i^p(\boldZ) \big) \, \big\}
    \end{gather}
    is satisfiable, where~$\boldX$ are the free variables in~${B \to H}$,
    and $\boldY$ and $\boldZ$ respectively are the free variables
    in formulas~$\Pnn{\Psi}{B}$ and~$\Pos{\Psi}{H}$ that are not in~$\boldX$.
    Note that, since~$\Pi$ is a program, it follows that~$H$ is a disjunction of the form~${H_1 \vee \dotsc \vee H_n}$ with~$p(\boldt) = H_i$ for some~$i \in \{1,\dotsc,n\}$.
    By construction, this implies that~$\Pos{\Psi}{H}$ is~${p(\boldt) \wedge \boldt = \boldY}$.
    Similarly, we have that~$B$ is a conjunction of the form~${B_1 \wedge \dotsc \wedge B_m}$ with~$q(\boldr) = B_j$ for some~$j \in \{1,\dotsc,m\}$.
    Therefore,
    there is a interpretation~$I$ that satisfies
    \begin{gather*}
    \exists \boldX\, \bigl( 
        B
        \wedge p(\boldt)  \wedge \pdef_j^q(\boldt) \wedge \pdef_i^p(\boldr)
    \bigr) 
    \end{gather*}
    Hence, there is also an edge from~$(p,\pdef_i)$ to~$(q,\pdef_j)$ in~$\G_{\pdefs}(\Pi)$.
\end{proof}

\begin{lemma}\label{lem:negative.for.programs}
    Let~$\Pi$ be a logic program and~$\pdef$ be a intensionality statement.
    If~$\Pi$ is negative on~$\pdef$, then it is~$\emptyset$\nobreakdash-negative on~$\pdef$.
\end{lemma}

\begin{proof}
    Pick any rule~$B \to H$ in~$\Pi$ and any strictly positive occurrence~$p(\boldt)$ in~$H$.
    Suppose, for the sake of contradiction, that~$\exists \boldX \boldY\, (B \wedge \SPos{H} \wedge \pdef^p(\boldY))$ is satisfiable.
    Since~$\Pi$ is a program, it follows that~$H$ is a disjunction of the form~$H_1 \vee \dotsc \vee H_n$ with~$p(\boldt) = H_i$ for some~$i \in \{1,\dotsc,n\}$.
    By construction, this implies that~$\Pos{\Psi}{H}$ is~${p(\boldt) \wedge \boldt = \boldY}$.
    Hence,
    there is a interpretation~$I$ that satisfies~$\exists \boldX\, \bigl(B
        \wedge p(\boldt)  \wedge \pdef^p(\boldt)
    \bigr)$.
    However, since~$\Pi$ is negative on~$\pdef$, it follows that~${\exists \boldX\, (B \wedge p(\boldt) \wedge \pdef^p(\boldt))}$ is unsatisfiable.
    Therefore, $\exists \boldX \boldY\, (B \wedge \SPos{H} \wedge \pdef^p(\boldY))$ is unsatisfiable and~$\Pi$ is~$\emptyset$\nobreakdash-negative on~$\pdef$.
\end{proof}

\begin{proofof}{Theorem~\ref{thm:splitting}}
Since~$\Pi$ is a program, by Lemmas~\ref{lem:positive.dependency.graph.for.programs} and~\ref{lem:negative.for.programs}, we get that
\begin{itemize}
    \item $\pdefs$ is separable on~$\G_{\pdefs,\empty}(\Pi)$; and
    \item each~$\Pi_i$ is negative on~$\pdef_j$ under context~$\Psi$  for all~$j \neq i$.
\end{itemize}
The result follows directly from Theorem~\ref{thm:theory.splitting} taking the empty context.
\end{proofof}